\documentclass[letterpaper,11pt]{article}
\usepackage[margin=1in]{geometry}
\usepackage[bookmarks, colorlinks=true, plainpages = false, citecolor = blue,linkcolor=red,urlcolor = blue, filecolor = blue,pagebackref]{hyperref}
\usepackage{url}
\usepackage{amsmath,amsfonts,amsthm,amssymb,bm, verbatim,dsfont,mathtools}
\usepackage{color,graphicx,appendix}
\usepackage{subfigure}
\usepackage{etoolbox}
\usepackage{tikz}
\usepackage{xr,xspace}
\usepackage{todonotes}
\usepackage{paralist}
\usepackage{caption,soul}
\usepackage{algorithm}
\usepackage{algorithmic}
\makeatletter

\theoremstyle{plain}
\newtheorem{theorem}{Theorem}
\newtheorem{lemma}{Lemma}

\newtheorem{corollary}{Corollary}
\theoremstyle{definition}

\newtheorem{remark}{Remark}
\newtheorem*{remark*}{Remark}

\newcommand{\Perp}{\perp}

\newcommand{\argmax}{\mathop{\arg\max}}

\usepackage{xspace,prettyref}

\newcommand{\Oprob}{O_{\Prob}}
\newcommand{\oprob}{o_{\Prob}}
\newcommand{\diverge}{\to\infty}

\newcommand{\iiddistr}{{\stackrel{\text{\iid}}{\sim}}}
\newcommand{\ones}{\mathbf 1}

\newcommand{\reals}{{\mathbb{R}}}

\newcommand{\naturals}{{\mathbb{N}}}

\newcommand{\eexp}{{\rm e}}

\newcommand{\identity}{\mathbf I}
\newcommand{\allones}{\mathbf J}

\newcommand{\diff}{{\rm d}}

\newcommand{\Expect}{\mathbb{E}}
\newcommand{\expect}[1]{\mathbb{E}\left[ #1 \right]}
\newcommand{\eexpect}[1]{\mathbb{E}[ #1 ]}

\newcommand{\Prob}{\mathbb{P}}
\newcommand{\pprob}[1]{ \mathbb{P}\{ #1 \} }
\newcommand{\prob}[1]{ \mathbb{P}\left\{ #1 \right\} }

\newcommand{\var}{\mathsf{var}}

\newcommand{\Bern}{{\rm Bern}}
\newcommand{\Binom}{{\rm Binom}}

\newcommand{\eg}{e.g.\xspace}
\newcommand{\ie}{i.e.\xspace}
\newcommand{\iid}{i.i.d.\xspace}
\newrefformat{eq}{(\ref{#1})}
\newrefformat{chap}{Chapter~\ref{#1}}
\newrefformat{sec}{Section~\ref{#1}}
\newrefformat{algo}{Algorithm~\ref{#1}}
\newrefformat{fig}{Fig.~\ref{#1}}
\newrefformat{tab}{Table~\ref{#1}}
\newrefformat{rmk}{Remark~\ref{#1}}
\newrefformat{clm}{Claim~\ref{#1}}
\newrefformat{def}{Definition~\ref{#1}}
\newrefformat{cor}{Corollary~\ref{#1}}
\newrefformat{lmm}{Lemma~\ref{#1}}
\newrefformat{prop}{Proposition~\ref{#1}}
\newrefformat{app}{Appendix~\ref{#1}}
\newrefformat{hyp}{Hypothesis~\ref{#1}}
\newrefformat{thm}{Theorem~\ref{#1}}

\newcommand{\pth}[1]{\left( #1 \right)}

\newcommand{\sth}[1]{\left\{ #1 \right\}}

\newcommand{\lnorm}[2]{\left\|{#1} \right\|_{{#2}}}

\newcommand{\fnorm}[1]{\|#1\|_{\rm F}}

\newcommand{\iprod}[2]{\left \langle #1, #2 \right\rangle}
\newcommand{\Iprod}[2]{\langle #1, #2 \rangle}
\newcommand{\indc}[1]{{\mathbf{1}_{\left\{{#1}\right\}}}}

\newcommand{\diag}[1]{\mathsf{diag} \left\{ {#1} \right\} }

\newcommand{\tY}{{\widetilde{Y}}}

\newcommand{\calG}{{\mathcal{G}}}

\newcommand{\calS}{{\mathcal{S}}}

\newcommand{\calY}{{\mathcal{Y}}}
\newcommand{\calZ}{{\mathcal{Z}}}

\newcommand{\Th}{{\rm th}}

\newcommand{\ML}{{\rm ML}\xspace}
\newcommand{\SDP}{{\rm SDP}\xspace}

\newcommand{\ER}{Erd\H{o}s-R\'enyi\xspace}

\newcommand{\IR}{\mathbb{R}}

\renewcommand{\hat}{\widehat}
\renewcommand{\tilde}{\widetilde}

\begin{document}

\title{Achieving Exact Cluster Recovery Threshold via Semidefinite
Programming: Extensions}
\date{\today}

\author{
Bruce Hajek \\
UIUC \\
{b-hajek@illinois.edu}
\and
Yihong Wu \\
UIUC\\
{yihongwu@illinois.edu}
\and
Jiaming Xu  \\
UIUC\\
{jxu18@illinois.edu}}

\author{ Bruce Hajek \and Yihong Wu \and Jiaming Xu\thanks{
B. Hajek and Y. Wu are with
the Department of ECE and Coordinated Science Lab, University of Illinois at Urbana-Champaign, Urbana, IL, \texttt{\{b-hajek,yihongwu\}@illinois.edu}.
J. Xu is with the Simons Institute for the Theory of Computing, University of California, Berkeley, Berkeley, CA, 
\texttt{jiamingxu@berkeley.edu}.
This work was in part presented at the Workshop on Community Detection, February 26-27,  2015, Institut Henri Poincar\'e, Paris.
The material was also presented in part at the 2015 49th Asilomar Conference on Signals, Systems and Computers \cite{HWX15-asilomar}
and the 2015 IEEE Information Theory Workshop \cite{HWX15-ITW}. 
This research was supported by the National Science Foundation under
Grant CCF 14-09106, IIS-1447879, NSF IOS 13-39388, and CCF 14-23088, and
Strategic Research
Initiative on Big-Data Analytics of the College of Engineering
at the University of Illinois, and DOD ONR Grant N00014-14-1-0823, and Grant 328025 from the Simons Foundation.
}}

\maketitle

\begin{abstract}
Resolving a conjecture of Abbe, Bandeira and Hall,
the authors have recently shown that the semidefinite
programming (SDP) relaxation of the maximum likelihood
estimator achieves the sharp threshold for exactly recovering
the community structure under the binary stochastic block
model of two equal-sized clusters.   The same was shown for the
case of a single cluster and outliers.    Extending the proof techniques,
in this paper it is shown that SDP relaxations also
achieve the sharp recovery threshold in the
following cases: (1) Binary stochastic block model with two
clusters of sizes proportional to network size but not necessarily
equal; (2) Stochastic block model with a fixed number of
equal-sized clusters; (3) Binary censored block model with
the background graph being Erd\H{o}s-R\'enyi.   Furthermore,
a sufficient condition is given for an SDP procedure to achieve
exact recovery for the general case of a fixed number of clusters plus outliers.
These results demonstrate the versatility of SDP relaxation as a simple, general purpose, computationally feasible methodology for community detection.
\end{abstract}

\section{Introduction}
The \emph{stochastic block model} (SBM) \cite{Holland83}, also known as
the \emph{planted partition model} \cite{Condon01},
 is a popular statistical model for studying the community detection and graph partitioning problem
(see, e.g.,~\cite{McSherry01,Decelle11,Mossel12, Mossel13, Massoulie13, Chen12,ChenXu14, AbbeSandon15} and the references therein).
In its simple form, it assumes that out of a total of $n$ vertices,
$(K_1 + \cdots + K_r) $ of them are partitioned into
$r$ clusters with sizes $K_1, \ldots , K_r,$ and the remaining
$n-(K_1 + \cdots + K_r)$ vertices do not belong to any clusters (called outlier vertices);
a random graph $G$ is then generated based on the cluster structure,  where
each pair of vertices is connected independently with probability $p$ if they are in the same cluster or $q$ otherwise.
In this paper, we focus on the problem of exactly recovering the clusters (up to a permutation of cluster indices) based on the graph $G$.

In the setting of two equal-sized clusters or a single cluster plus outlier vertices,
recently it has been shown in \cite{HajekWuXuSDP14} that the semidefinite programming (SDP) relaxation of the
maximum likelihood (ML) estimator  achieves the optimal
recovery threshold with high probability, in the asymptotic regime
of $p=a \log n /n$ and $q=b\log n /n$ for fixed constants $a,b$ and cluster sizes growing linearly in $n$
as $n \to \infty$.
The result for two equal-sized clusters was originally conjectured in \cite{Abbe14}
and another resolution was recently given
in \cite{Bandeira15} independently.

In this paper, we extend the optimality of SDP to the following three cases, while still assuming $p=a \log n /n$ and $q=b\log n /n$ with $a>b>0$:
 \begin{itemize}
\item  Stochastic block model with two asymmetric clusters: the first cluster consists of  $K$ vertices and
the second cluster consists of $n-K$ vertices with $K=\lfloor \rho n \rfloor $ for some $\rho \in [0,1/2].$
The value of $\rho$ may be known or unknown to the recovery procedure.
 \item Stochastic block model with $r$ clusters of equal size $K$: $r \ge 2$ is a fixed integer and $n=rK$.
 \item Censored block model with two clusters:
  given an \ER random graph $G \sim \calG(n, p)$,
 each edge $(i,j)$ has a label $L_{ij} \in \{\pm 1\}$ independently drawn according to the distribution:
\begin{align*}
P_{L_{ij} | \sigma^\ast_i, \sigma^\ast_j } = (1-\epsilon) \indc{L_{ij} = \sigma_i^\ast \sigma^\ast_j} +  \epsilon \indc{L_{ij} = - \sigma_i^\ast \sigma^\ast_j},
\end{align*}
where $\sigma^\ast_i=1$ if vertex $i$ is in the first cluster and $\sigma^\ast_i=-1$ otherwise; $\epsilon \in [0, 1/2]$ is a fixed constant.\footnote{Under the censored block model, the graph itself does not contain any information about the underlying clusters and we are interested in recovering the clusters by observing the graph and edge labels.}
  \end{itemize}

In all three cases, we show that a necessary condition for the maximum likelihood (ML) estimator to succeed coincides with a sufficient condition for the correctness of the \SDP procedure, thereby establishing both the optimal recovery threshold and the optimality of the SDP relaxation. The proof techniques in this paper are similar to those in \cite{HajekWuXuSDP14}; however, the construction and validation of dual certificates for the success of \SDP are more challenging especially in the multiple-cluster case.
Notably, we resolve an open problem raised in \cite[Section 6]{AbbeBAndeiraSigner14} about the optimal recovery threshold in the censored block model and show that the optimal recovery threshold can be achieved in polynomial-time via \SDP.

 To further investigate the applicability of  \SDP procedures for community
 detection, we explored two cases for which the algorithm is \emph{adaptive} to the unknown
 cluster sizes.  First,
 we found that for two clusters,
 the conditions for exact recovery are the
 strongest in the equal-sized case.
 This suggests, and it is shown in Section \ref{subset:unkown_size}, that
 if the cluster size constraint is replaced by an appropriate Lagrangian term  not
 depending on the cluster size,
exact recovery is achieved for all cluster sizes under the condition required for two equal-sized clusters.
Secondly,  we examined the general community
 detection problem with
a fixed number of unequal-sized clusters and with outlier vertices, and
 identified a sufficient condition for  the  \SDP procedure to achieve exact
 recovery with knowledge  of only the smallest cluster size and the
 parameters $a,b.$ (See \prettyref{sec:general_case}.)

The optimality result of SDP has recently been extended to the cases with $o(\log n)$ number of equal-sized clusters in  \cite{ABBK}
 and  a fixed number of  clusters with unequal  sizes in \cite{perry2015semidefinite}.

 \paragraph{Parallel independent work}
 The exact recovery problem in the logarithmic
 sparsity regime has been independently studied in \cite{AbbeSandon15}
 in a more general setting: Given a fixed $r \times r$ matrix $Q$ and a probability vector $s=(s_1,\ldots,s_r)$, the cardinality of the $k^\Th$ community is assumed to be $s_k n$ and vertices in the $k^\Th$ and $\ell^\Th$ community are connected independently with
 probability $Q_{k\ell} \log n /n$. The optimal recovery
 threshold is obtained as a function of $Q$ and $s$. In the special setting of $Q_{kl}=a$ if $k=\ell$ and $Q_{kl}=b$ if $k \neq \ell$, for
two asymmetric clusters or multiple equal-sized clusters,
their general optimal threshold reduces to  those derived in this paper.
Assuming full knowledge of the parameters $Q$ and $s$, the optimal
recovery threshold is further shown in \cite{AbbeSandon15} to be achievable in $o(n^{1+\epsilon})$ time
 for all $\epsilon>0$ via a two-phase procedure, consisting of a partial recovery algorithm followed by a cleanup step.

 For the case of $r$ equal-sized clusters, it is  independently shown in \cite{YunProutiere14}  that 
 the optimal recovery threshold 
can be obtained in polynomial-time. Their clustering algorithm is a two-step procedure similar to \cite{AbbeSandon15}, where the partial recovery is achieved via a simple spectral algorithm.
For the case with two unequal-sized clusters, a sufficient (but not tight) recovery condition is also derived in  \cite{YunProutiere14}. 



\paragraph{Further literature on SDP for cluster recovery}
 There has been a recent surge of interest in analyzing the semidefinite programming
relaxation approach for cluster recovery;
some of the latest development are summarized below.
For different recovery approaches such as spectral methods,
we refer the reader to \cite{ChenXu14, AbbeSandon15} for details.

The SDP approach is mostly analyzed
in the regime where the average degrees scale as
$\log n$, with the objective of exact cluster recovery. In this setting, the analysis often relies on the standard technique of dual witnesses, which amounts to
constructing the dual variables so that the desired KKT conditions are satisfied for the primal variable corresponding to the true clusters.
The SDP has been applied to recover cliques or densest subgraphs
in \cite{ames2011plantedclique,ames2010kclique,ames2013robust}.
For the stochastic block model with possibly unbounded number of clusters,
a sufficient condition for an SDP procedure to achieve exact recovery is obtained in \cite{ChenXu14},
which improves the sufficient conditions in \cite{Chen12,oymak2011finding} in terms of scaling.
Various formulations of SDP  for cluster recovery are discussed in  \cite{Amini14}.
The robustness of the SDP  has been investigated in \cite{FK01} for minimum bisection in the semirandom model with monotone adversary and, more recently,
in \cite{cai2014robust} for generalized SBM with arbitrary outlier vertices.
 The SDP machinery has also been applied to recover clusters with partially observed graphs \cite{Chen13,Vinayak14} and
 binary matrices \cite{Hajek13}.  In the converse direction, necessary conditions for the
success of particular SDPs are obtained in \cite{vinayak2013sharp,ChenXu14}. In contrast to the previous work mentioned
above where the constants are often loose, the recent line of work initiated by \cite{Abbe14,AbbeBAndeiraSigner14}, and followed by
 \cite{HajekWuXuSDP14,Bandeira15} and the current paper, focus on establishing necessary and sufficient conditions in the special case of a fixed number of clusters with sharp constants, attained via SDP relaxations.

In the sparse graph case with bounded average degree, exact recovery is provably impossible and instead the goal is to achieve partial recovery, namely, to correctly cluster all but a small fraction of vertices. Using Grothendieck's inequality, a sufficient condition for SDP to achieve partial recovery is
 obtained in \cite{Vershynin14};
  the technique is extended
 to the labeled stochastic block model
 in \cite{LelargeMassoulieXu15}.
In \cite{MontanariSen15}, an SDP-based test is applied to distinguish the binary symmetric stochastic block model versus
 the \ER random graph  and shown to attain the optimal detection threshold. 
 

\paragraph{Notation}
Denote the identity matrix by $\identity$, the all-one matrix by $\allones$ 
and the all-one vector by $\mathbf{1}$. 
We write  $X \succeq 0$ if $X$ is symmetric and positive semidefinite and $X \ge 0$ if all the entries of $X$ are non-negative.
Let $\calS^n$ denote the set of all $n \times n$ symmetric matrices. For $X \in \calS^n$, let $\lambda_2(X)$ denote its second smallest eigenvalue.
For any matrix $Y$, let $\|Y\|$ denote its spectral norm.
For any positive integer $n$, let $[n]=\{1, \ldots, n\}$.
For any set $T \subset [n]$, let $|T|$ denote its cardinality and $T^c$ denote its complement.
For $\rho \in [0,1]$, let $\bar \rho=1-\rho$.
We use standard big $O$ notations,
e.g., for any sequences $\{a_n\}$ and $\{b_n\}$, $a_n=\Theta(b_n)$
if there is an absolute constant $c>0$ such that $1/c\le a_n/ b_n \le c$; $a_n =\Omega(b_n)$ or $b_n = O(a_n)$ if there exists  an absolute constant $c>0$ such that $a_n/b_n \ge c$.
Let $\Bern(p)$ denote the Bernoulli distribution with mean $p$ and
$\Binom(n,p)$ denote the binomial distribution with $n$ trials and success probability $p$.
All logarithms are natural and we use the convention $0 \log 0=0$.

\section{Binary asymmetric SBM}
\label{sec:asym}

\subsection{Known cluster size}  \label{subset:kown_size}

Let $A$ denote the adjacency matrix of the graph, and $(C^\ast_1, C^\ast_2)$ denote the
underlying true partition, where the clusters $C^\ast_1$ and $C^\ast_2$ have cardinalities
$K$ and $n-K$, respectively, and we consider the asymptotic regime $K=\lceil n\rho \rceil$
as $n\to\infty$ for $\rho \in [0,\frac{1}{2}]$ fixed.
In this subsection we assume that $\rho$ is known to the recovery procedure and the goal is to obtain the $\rho$-dependent optimal recovery threshold attained by SDP relaxations.

The cluster structure  under the binary  stochastic block model can be represented by a vector
$\sigma \in \{\pm 1\}^n$ such that $\sigma_i=1$ if vertex $i$ is
in the first cluster and $\sigma_i=-1$ otherwise. Let $\sigma^\ast$ correspond to the true clusters.
Then the \ML estimator of $\sigma^\ast$  for the case $a>b$ can be simply stated as
\begin{align}
\max_{\sigma}  & \; \sum_{i,j} A_{ij} \sigma_i \sigma_j \nonumber  \\
\text{s.t.	}  & \; \sigma_i \in \{\pm1\}, \quad i \in [n] \nonumber \\
 & \; \sigma^\top \mathbf{1} =2K-n,
 \label{eq:SBMML1_unbalanced}
\end{align}
which maximizes the number of in-cluster edges minus the number of out-cluster edges subject to the cluster size constraint.
If $K=n/2$,  \prettyref{eq:SBMML1_unbalanced} reduces to the minimum graph bisection problem which is NP-hard in the worst case. Due to the computational intractability of the ML estimator, next we turn to its convex relaxation.
Let $Y=\sigma \sigma^\top$. Then $Y_{ii}=1$ is equivalent to $\sigma_i = \pm 1$, and
$\sigma^\top \mathbf{1}=\pm (2K-n)$ if and only if $\Iprod{Y}{\allones}=(2K-n)^2$.
Therefore, \prettyref{eq:SBMML1_unbalanced} can be recast as\footnote{Henceforth, all matrix variables in the optimization are symmetric.}
\begin{align}
\max_{Y}  & \; \Iprod{A}{Y} \nonumber  \\
\text{s.t.	} & \; \text{rank}(Y)=1  \nonumber  \\
& \;  Y_{ii} =1, \quad i \in [n]\nonumber \\
& \;  \Iprod{\allones}{Y} =(2K-n)^2. \label{eq:SBMML2_unbalanced}
\end{align}
Notice that any feasible solution
is a rank-one positive semidefinite matrix. Relaxing this
condition by dropping the rank-one restriction, we obtain the following convex relaxation of \prettyref{eq:SBMML2_unbalanced},
which is a semidefinite program:
\begin{align}
\widehat{Y}_{\SDP} = \argmax_{Y}  & \; \langle A, Y \rangle \nonumber  \\
\text{s.t.	} & \; Y \succeq 0  \nonumber \\
 & \;  Y_{ii} =1, \quad i \in [n] \nonumber \\
 & \; \Iprod{\allones}{Y} =(2K-n)^2. \label{eq:SBMconvex_unbalanced}
\end{align}
We note that
the only model parameter needed by the estimator \prettyref{eq:SBMconvex_unbalanced} is the cluster size $K$.

Let $Y^\ast=\sigma^\ast (\sigma^\ast)^\top$ correspond to the true partition 
and $\calY_n \triangleq \{\sigma \sigma^\top: \sigma \in \{ \pm 1 \}^n, \sigma^\top \mathbf{1} = 2K -n \}$ denote
the set of all admissible partitions. 
The following result establishes the optimality of the \SDP procedure.

\begin{theorem}\label{thm:SBMSharp_unbalanced}
If $\eta( \rho, a, b) > 1$, then $\min_{Y^* \in \calY_n} \pprob{\widehat{Y}_{\SDP}=Y^\ast} \geq  1 - n^{-\Omega(1)}$ as $n \to \infty$, 
where 
\begin{align}
\eta(\rho, a, b) = \frac{a+b}{2} -\gamma +  \frac{ ( \bar{\rho}  - \rho ) \tau}{2 } \log \frac{\rho  (\gamma+( \bar{\rho} -\rho ) \tau) } {\bar{\rho} (\gamma- ( \bar{\rho} -\rho ) \tau)}
\label{eq:threshold}
\end{align}
with  $\bar{\rho}=1-\rho$, 
$\tau = \frac{a-b}{\log a - \log b}$, $\gamma= \sqrt{ (\bar{\rho} -\rho )^2 \tau^2 + 4 \rho \bar{\rho} a b }$, and $\eta(0, a, b)= \lim_{\rho \to 0} \eta(\rho, a, b)= \frac{a+b}{2} - \tau \log  \frac{ \eexp \sqrt{ab} }{\tau}.$
\end{theorem}

The proof of \prettyref{thm:SBMSharp_unbalanced} is similar in outline to the
proof  given in  \cite{HajekWuXuSDP14}, but a considerable detour is needed to handle
the imbalance.   Notice that by definition, $\eta(\rho, a, b) = \eta(\bar{\rho}, a, b)$, and 
$\eta(1/2, a, b) =  \frac{1}{2} ( \sqrt{a}-\sqrt{b} )^2$. The threshold function $\eta(\rho, a, b)$ 
turns out to be the error exponent in  the following large deviation events. For vertex $i$, let $e(i, C^*_1)$
denotes the number of edges between vertex $i$ and vertices in $C^*_1$, and define $e(i, C^*_2)$ similarly. 
Then,
 \begin{align*}
 \prob{e(i, C^*_1) - e(i, C^*_2) \le \tau( \rho -\bar{\rho} ) \log n  }  =n^{-\eta(\rho, a, b) +o(1)}, \quad \forall i \in C^*_1, \\
  \prob{e(i, C^*_2) - e(i, C^*_1) \le \tau( \bar{\rho} -\rho ) \log n  }  =n^{-\eta(\rho, a, b) +o(1)}, \quad \forall i \in C^*_2.
 \end{align*}

Next we prove a converse for \prettyref{thm:SBMSharp_unbalanced} which shows that the
recovery threshold achieved by the \SDP relaxation is in fact optimal.

\begin{theorem} \label{thm:SBMSharpConverse_unbalanced}
If $
\eta(\rho,a,b)<1
$
and $\sigma^\ast$ is uniformly chosen over $\{ \sigma \in \{ \pm 1\}^n: \sigma^\top \mathbf{1} =2 K -n \}$,
then for any sequence of estimators $\widehat{Y}_n$, $\pprob{\widehat{Y}_n = Y^\ast } \to 0$.
\end{theorem}

In the special case with two equal-sized clusters, we have $K=n/2$ and
$\eta(1/2, a, b) =  \frac{1}{2} ( \sqrt{a}-\sqrt{b} )^2$. The corresponding threshold $(\sqrt{a}-\sqrt{b})^2 > 2$ has been established in \cite{Abbe14,Mossel14}, and the achievability by SDP has been shown in \cite{HajekWuXuSDP14} and independently
 by \cite{Bandeira15} later. 

A recent work \cite{YunProutiere14} also studies the exact recovery
problem in the unbalanced case and provides the sufficient (but not tight) recovery condition for a polynomial-time two-step procedure based on the spectral method.

\subsection{Unknown cluster size}  \label{subset:unkown_size}
\prettyref{thm:SBMSharp} shows that if one knows the relative cluster size $\rho$, the SDP relaxation \prettyref{eq:SBMconvex_unbalanced} achieves the size-dependent optimal threshold $\eta(\rho,a,b)>1$. For
fixed $a$ and $b$, $\eta( \rho, a, b)$ is minimized at $\rho=\frac{1}{2}$ (see \prettyref{app:eta} for a proof). This suggests that for two communities the equal-sized case is the most difficult to cluster. Indeed, the next result proves that if there is no constraint on the cluster size, then the optimal recovery threshold coincides with that in the balanced case, \ie, $(\sqrt{a}-\sqrt{b})^2 > 2$, which can be achieved by a penalized SDP.
\begin{theorem}
	Let
	\begin{align}
\widehat{Y}_{\SDP}' = \argmax_{Y}  & \; \langle A , Y \rangle - \lambda^* \langle \allones, Y \rangle \nonumber  \\
\text{s.t.	} & \; Y \succeq 0  \nonumber \\
 & \;  Y_{ii} =1, \quad i \in [n]. \label{eq:SDP2}
\end{align}
where $\lambda^* = \tau \frac{\log n}{n}$ and $\tau = \frac{a-b}{\log a-\log b}$. If $(\sqrt{a}-\sqrt{b})^2 > 2$, then $\min_{Y^* \in \calY_n'} \pprob{\widehat{Y}_{\SDP}'=Y^\ast} \ge 1- n^{-\Omega(1) }$ as $n \to \infty$,
where $\calY_n' \triangleq \{\sigma \sigma^\top: \sigma \in \{\pm 1\}^n \}$.
	\label{thm:SDP2}
\end{theorem}

\begin{remark}
\prettyref{thm:SDP2} holds for all cluster sizes $K,$ including the extreme case where the entire network forms a single cluster ($K=0$), in which case the SDP \prettyref{eq:SDP2} outputs $Y^*=\allones$ with high probability. The downside is that the penalization parameter $\lambda^*$ depends on the parameters $a$ and $b$. Nevertheless, there exists a fully data-driven choice of $\lambda^*$ based on the degree distribution of the network, so that \prettyref{thm:SDP2} continues to hold whenever the cluster sizes scale linearly,
\ie, $K/n \to \rho \in (0,\frac{1}{2}]$; the price to pay for adaptivity is that the probability of error vanishes polylogarithmically instead of polynomially as $n\diverge$.
	See \prettyref{app:thmSDP2} for details.
	\label{rmk:SDP2}
\end{remark}

\section{SBM with multiple equal-sized clusters}
\label{sec:rary}

The cluster structure under the stochastic block model
with $r$ clusters of equal size $K$ can be represented by $r$ binary vectors $\xi_1,
\ldots, \xi_r \in \{0,1\}^n$, where $\xi_k$ is  the indicator function
of the cluster $k$, such that $\xi_k(i)=1$ if vertex $i$ is in cluster $k$
and $\xi_k(i)=0$ otherwise.
 Let $\xi^\ast_1, \ldots, \xi^\ast_r$ correspond to the true clusters and let $A$ denote the adjacency matrix.
Then the  maximum likelihood  (\ML) estimator of $\xi^\ast$  for the case $a>b$ can be simply stated as
\begin{align}
\max_{\xi}  & \; \sum_{i,j} A_{ij}  \sum_{k=1}^r \xi_k(i) \xi_k(j) \nonumber  \\
\text{s.t.	}  & \; \xi_k \in \{0,1 \}^n, \quad k \in [r] \nonumber \\
 & \; \xi_k^\top \mathbf{1} =K, \quad k \in [r] \nonumber \\
 & \; \xi_k^\top \xi_{k'}= 0, \quad k \neq k',
 \label{eq:SBMML1}
\end{align}
which maximizes the number of in-cluster edges. 
Alternatively, one can encode the cluster structure from the vertices' perspective. 
Each vertex $i$ is associated with a vector $x_i$ which is allowed
to be one of the $r$ vectors $v_1, v_2, \ldots, v_r  \in \reals^{r-1} $ defined as follows:
Take an equilateral simplex $\Sigma_r$ in $\reals^{r-1}$ with vertices
$v_1, v_2, \ldots, v_r$  such that $\sum_{k=1}^r v_k =0$ and $\| v_k \|=1$
for $1 \le k \le r$.  Notice that
$\Iprod{v_k}{v_{k'}} = -1/(r-1)$ for $k \neq k'$.
Therefore,  the  \ML estimator  given in \prettyref{eq:SBMML1} can be recast as
\begin{align}
\max_{x}  & \; \sum_{i,j} A_{ij}  \Iprod{x_i}{ x_j} \nonumber  \\
\text{s.t.	}  & \; x_i \in \{v_1, v_2, \ldots, v_r\}, \quad i \in [n] \nonumber \\
& \; \sum_i x_i =0 .
 \label{eq:SBMML2}
\end{align}
When $r=2$, the above program includes the NP-hard minimum graph bisection problem as a special case. 
Let us consider its convex relaxation similar to the
\SDP relaxation studied by Goemans and Williamson \cite{Goemans95} for MAX CUT and by Frieze and Jerrum
\cite{FriezeJerrum97} for MAX $k$-CUT and MAX BISECTION.
To obtain an SDP relaxation, we replace $x_i$ by $y_i$ which is allowed to be
any unit vector in $\reals^n$ under the constraint $\Iprod{y_i}{y_j} \ge -1/(r-1)$
and $\sum_i y_i =0$. Defining $Y \in \reals^{n \times n}$ such that $Y_{ij}= \Iprod{y_i}{y_j}$, we obtain an \SDP:
\begin{align}
\max_{Y}  & \; \Iprod{A}{Y} \nonumber  \\
\text{s.t.	} & \; Y \succeq 0, \nonumber  \\
& \;  Y_{ii} =1, \quad i \in [n] \nonumber \\
& \;  Y_{ij} \ge -1/(r-1), \quad i, j \in [n]  \nonumber \\
& \; Y \mathbf{1} = 0. 
\label{eq:SBMML3}
\end{align}
We remark that we could as well have worked with the constraint $\Iprod{Y}{\allones} =0$,
which, for $Y \succeq 0$, is equivalent to the last constraint in \prettyref{eq:SBMML3}.  
Letting $Z= \frac{r-1}{r} Y + \frac{1}{r} \allones$, we can also equivalently rewrite
\prettyref{eq:SBMML3} as
\begin{align}
\widehat{Z}_{\SDP} =
\argmax_Z & \; \Iprod{A}{Z} \nonumber  \\
\text{s.t.	} & \; Z \succeq 0, \nonumber  \\
& \;  Z_{ii} =1, \quad i \in [n] \nonumber \\
& \;  Z_{ij} \ge 0, \quad i, j \in [n]  \nonumber \\
& \; Z \mathbf{1} =K \mathbf{1}. 
\label{eq:SDP_RZ}
\end{align}

The only model parameter needed by the estimator \prettyref{eq:SDP_RZ} is the cluster size $K$.
Let $Z^*=\sum_{k=1}^r \xi_k^*(\xi_k^*)^\top$ correspond to the true clusters and
define
\begin{align*}
\calZ_{n,r} = \sth{\sum_{k=1}^r \xi_k \xi_k^\top: \xi_k \in \{ 0, 1 \}^n , \; \xi_k^\top \mathbf{1} =K, \;  \xi_k^\top \xi_{k'}= 0, \; k \neq k' }.
\end{align*}
The sufficient condition
for the success of \SDP in \prettyref{eq:SDP_RZ} is given as follows.
\begin{theorem}\label{thm:SBMSharp}
If $\sqrt{a} - \sqrt{b} > \sqrt{r}$, then $\min_{Z^* \in {\cal Z}_{n,r}} \pprob{\widehat{Z}_{\SDP}=Z^\ast} \geq 1 - n^{-\Omega(1)}$ as $n \to \infty$.
\end{theorem}

The following result establishes the optimality of the \SDP procedure.
\begin{theorem}\label{thm:converse}
If $\sqrt{a} - \sqrt{b} < \sqrt{r}$ and the clusters are uniformly chosen at random
among all $r$-equal-sized partitions of $[n]$,
then for any sequence of estimators $\widehat{Z}_n$,
$ \pprob{\widehat{Z}_n=Z^\ast} \to 0$ as $n \to \infty$.
\end{theorem}

The optimal recovery threshold $\sqrt{a}-\sqrt{b}= \sqrt{r}$ is also obtained by two parallel independent works \cite{YunProutiere14,AbbeSandon15}
via a polynomial-time two-step procedure,  consisting of a partial recovery algorithm followed by a cleanup stage. The previous work \cite{ChenXu14}
studies the stochastic block model in a much more general setting with $r$  clusters of equal size $K$ plus outlier vertices,
 where $r, K$ and the edge probabilities $p,q$ may scale with $n$ arbitrarily as long as $r K \le n$; it is shown  that an SDP achieves exact recovery with high probability provided that
\begin{equation}
K^2(p-q)^2 \geq C \left( Kp(1-q) \log n + q(1-q)  n \right)
	\label{eq:CX14}
\end{equation}
 for some universal constant $C$.
In the special setting where the network consists of a fixed number of clusters without outliers and $p= a \log n/n> q= b \log n/n$,
the sufficient condition \prettyref{eq:CX14} simplifies to $\sqrt{a}-\sqrt{b} \geq C' \sqrt{r}$ for some absolute constant $C'$, which is
off by a constant factor compared to the sharp  sufficient condition $\sqrt{a}-\sqrt{b}>\sqrt{r}$ given by \prettyref{thm:SBMSharp}.

It is straightforward to extend the current proof of \prettyref{thm:converse} to the regime where  $r= \gamma \log^{s} n$, $p= \frac{a \log^{s+1} n}{n}>q= \frac{b \log^{s+1} n}{n}$
for any fixed $\gamma,a,b >0$ and $s \in [0,1)$, showing that SDP achieves the optimal recovery threshold $\sqrt{a}-\sqrt{b}= \sqrt{\gamma}$. 
Indeed, the preprint \cite{ABBK} shows similar optimality results of SDP for $r=o(\log n)$ number of equal-sized clusters. 
Conversely, it has been recently proved in \cite{HajekWuXu_one_sdp15} that 
SDP relaxations cease to be optimal for logarithmically many communities in the sense that 
SDP is constantwise suboptimal when $r \ge C \log n$ for a large enough constant $C$
and orderwise suboptimal when $r=\omega(\log n)$. 

\section{Binary censored block model}
\label{sec:cbm}
Under the binary censored block model, with possibly unequal cluster sizes,
 the cluster structure  can be represented by a vector
$\sigma \in \{\pm 1\}^n$ such that $\sigma_i=1$ if vertex $i$ is
in the first cluster and $\sigma_i=-1$ if vertex $i$ is in the second cluster.
Let $\sigma^\ast \in \{\pm 1\}^n$ correspond to the true clusters.
Let $A$ denote the weighted adjacency matrix such that $A_{ij}=0$ if $i, j$ are not connected
by an edge; $A_{ij}=1$ if $i,j$ are connected by an edge with label $+1$; $A_{ij}=-1$ if $i,j$ are
connected by an edge with label $-1$.
Then the \ML estimator of $\sigma^\ast$ can be simply stated as
\begin{align}
\max_{\sigma}  & \; \sum_{i,j} A_{ij} \sigma_i \sigma_j \nonumber  \\
\text{s.t.	}  & \; \sigma_i \in \{\pm1\}, \quad i \in [n],  \label{eq:SBMML1_labeled}
\end{align}
which maximizes the number of  in-cluster $+1$ edges minus that of  in-cluster $-1$ edges,
or equivalently, maximizes the number of  cross-cluster $-1$ edges minus that of  cross-cluster $+1$ edges.
The NP-hard max-cut problem can be reduced to \prettyref{eq:SBMML1_labeled} by simply labeling all the edges
in the input graph as $-1$ edges, and thus  \prettyref{eq:SBMML1_labeled}
is computationally intractable in the worst case.
Instead, we consider the SDP studied in \cite{AbbeBAndeiraSigner14} obtained by convex relaxation.
Let $Y=\sigma \sigma^\top$. Then $Y_{ii}=1$ is equivalent to $\sigma_i = \pm 1$.
Therefore, \prettyref{eq:SBMML1} can be recast as
\begin{align}
\max_{Y}  & \; \Iprod{A}{Y} \nonumber  \\
\text{s.t.	} & \; \text{rank}(Y)=1 \nonumber  \\
& \;  Y_{ii} =1, \quad i \in [n]. \label{eq:SBMML2_labeled}
\end{align}
Replacing the rank-one constraint by positive semidefiniteness,
we obtain the following convex relaxation of \prettyref{eq:SBMML2_labeled},
which is an SDP:
\begin{align}
\widehat{Y}_{\SDP} = \argmax_{Y}  & \; \langle A, Y \rangle \nonumber  \\
\text{s.t.	} & \; Y \succeq 0  \nonumber \\
 & \;  Y_{ii} =1, \quad i \in [n]. \label{eq:SBMconvex_labeled}
\end{align}
We remark that \prettyref{eq:SBMconvex_labeled} does not rely on any knowledge of the model parameters.
Let $Y^\ast=\sigma^\ast (\sigma^\ast)^\top$ and $\calY_n \triangleq \{\sigma \sigma^\top\colon \sigma \in \{\pm 1\}^n\}$.
The following result establishes the success condition of the \SDP procedure in the scaling regime $p=a\log n /n$ for a fixed constant $a$:
\begin{theorem}\label{thm:SBMSharp_labeled}
If $ a ( \sqrt{1-\epsilon}-\sqrt{\epsilon} )^2> 1$, then $\min_{Y^* \in \calY_n} \pprob{\widehat{Y}_{\SDP}=Y^\ast} \geq 1 - n^{-\Omega(1)}$ as $n \to \infty$.
\end{theorem}

Next we prove a converse for \prettyref{thm:SBMSharp_labeled} which shows that the recovery threshold achieved by the \SDP relaxation is in fact optimal.
\begin{theorem} \label{thm:PlantedSharpConverse_labeled}
If $ a ( \sqrt{1-\epsilon}-\sqrt{\epsilon} )^2 < 1$ and $\sigma^\ast$  is uniformly chosen from $ \{\pm 1\}^n$,
then for any sequence of estimators $\widehat{Y}_n$, $\pprob{\widehat{Y}_n = Y^\ast } \to 0$ as $n \to \infty$.
\end{theorem}
\prettyref{thm:PlantedSharpConverse_labeled} still holds if the cluster sizes are proportional to $n$ and known to
the estimators, \ie,  the prior distribution of $\sigma^\ast$ is uniform over $\{ \sigma \in \{\pm 1\}^n: \sigma^\top \mathbf{1} = 2K-n\}$ for $K=\lfloor \rho n \rfloor $
with $\rho \in (0,1/2]$.  

Denote by $a^*(\epsilon)$ the optimal recovery threshold, namely, the infimum of $a>0$ such that exact cluster recovery is possible with probability converging to one as $n\diverge$. Our results show that for all $\epsilon \in [0,1/2]$, the optimal recovery threshold is given by
\begin{equation}
a^*(\epsilon) = \frac{1}{( \sqrt{1-\epsilon} - \sqrt{\epsilon} )^2}, 
	\label{eq:astar}
\end{equation}
and can be achieved by the SDP relaxations.
The optimal recovery threshold is insensitive to $\rho$, 
which is in contrast to what we have seen for the binary stochastic block model.

Exact cluster recovery in the censored block model
 is previously studied in \cite{AbbeBAndeiraSigner14} and  it is shown that if $\epsilon \to 1/2$,
 the maximum likelihood estimator achieves the optimal recovery threshold $a (1-2\epsilon)^2>2 + o(1)$,
 while an SDP relaxation of the ML estimator succeeds if  $a (1-2\epsilon)^2> 4+o(1)$.
The optimal recovery threshold for any fixed $\epsilon \in (0,1/2)$ and whether
it can be achieved in polynomial-time were previously unknown. \prettyref{thm:SBMSharp_labeled} and \prettyref{thm:PlantedSharpConverse_labeled} together
show that the \SDP relaxation achieves the optimal recovery threshold  $ a ( \sqrt{1-\epsilon} - \sqrt{\epsilon} )^2>1$ for any fixed constant $\epsilon \in [0, 1/2]$.
Notice that $( \sqrt{1-\epsilon} - \sqrt{\epsilon} )^2 = \frac{1}{2}(1-2 \epsilon)^2 + o((1-2\epsilon)^2)$ when $\epsilon \to 1/2$.
For the censored block model with the background graph being random regular graph, it is further shown in \cite{HWX15-ITW} that the
SDP relaxations also achieve the optimal exact recovery threshold.

The above exact recovery threshold in the regime $p=a \log n /n$ shall be contrasted with the positively correlated
recovery threshold in the sparse regime $p=a/n$ for constant $a$. In this sparse regime, there exists at least a constant fraction of vertices with
no neighbors and exactly recovering the clusters is hopeless; instead, the goal is to find an estimator $\widehat{\sigma}$
positively correlated with $\sigma^\ast$ up to a global flip of signs. It was conjectured in \cite{Heimlicher12} that
the positively correlated recovery is possible if and only if $a(1-2\epsilon)^2>1$; the converse part is
shown in \cite{LelargeMassoulieXu15} and recently it is proved in \cite{SKLZ15} that spectral algorithms achieve the sharp threshold in polynomial-time.

\section{An SDP for general cluster structure}
\label{sec:general_case}

In this section we consider SDPs for the general case of multiple clusters
and outliers.  We assume there are $r$ clusters with sizes $K_1, \ldots , K_r,$ and
$n-(K_1 + \cdots + K_r)$ outlier vertices.  Vertices in the same cluster are connected  with probability $p$, while other pairs of vertices are connected
between them with probability $q.$   We consider the asymptotic regime
$p=\frac{a\log n}{n},$ $q=\frac{b\log n}{n}$  and   $K_k = \rho_k n$ as $n\rightarrow \infty$
for $a, b, \rho_0, \ldots  , \rho_r$ fixed, with $\rho_1 \geq \ldots \geq \rho_r>0.$
Let $\rho_{\min}=\rho_r.$
   We derive sufficient conditions for exact recovery by SDPs.
While the conditions are not the tightest possible
for specific cases, we would like to identify an algorithm that
recovers the cluster matrix exactly without knowing the details of the
cluster structure.   As in \prettyref{sec:rary}, the true cluster matrix can be expressed as
$Z^*=\sum_{k=1}^r \xi_k^*(\xi_k^*)^\top,$ where $\xi_k^*$ is the indicator
function of the $k^{\Th}$ cluster.
Denote by $\calZ_n$ the collection of all such cluster matrices.

Consider the SDP
\begin{eqnarray}
&\widehat{Z}_{\SDP}  = \arg\max~~ \langle A, Z \rangle   \label{eq.SDP_RZ} \\
&~~~~~~~~~~~~~~~~~~~~~~~~~~~~~~~~\text{s.t. } \begin{array}[t]{l}
Z \succeq 0 \\
Z_{ii} \leq 1 \\
Z_{ij} \geq 0 \\
\langle \mathbf{I}, Z \rangle = K_1 + \cdots + K_r \\
\langle \mathbf{J}, Z \rangle = K_1^2 + \cdots + K_r^2 . 
\end{array}  \nonumber
\end{eqnarray}
Implementing the SDP \eqref{eq.SDP_RZ} requires no knowledge of the density parameters $a$ and $b$, the number of clusters $r,$  or the sizes of the individual clusters; but it does require the exact knowledge of the sum as well as the sum of squares of the cluster sizes, which, in practical applications, may be  unrealistic to assume. Therefore, similar to \prettyref{eq:SDP2}, we also consider
the following penalized SDP, obtained by removing the constraints for those two quantities while augmenting the objective function:
\begin{eqnarray}
&\widehat{Z}_{\SDP}  = \arg\max~~ \langle A, Z \rangle   -\eta^* \langle \mathbf{I}, Z \rangle  - \lambda^* \langle \mathbf{J}, Z \rangle   \label{eq.SDP_RZaug} \\
&~~~\text{s.t. } \begin{array}[t]{l}
Z \succeq 0 \\
Z_{ii} \leq 1 \\
Z_{ij} \geq 0. \\
\end{array}  \nonumber
\end{eqnarray}
Here the penalization parameters $\eta^*$ and $\lambda^*$ must be specified.

Clearly the above two SDPs are different and need not have the same solutions; nevertheless, they are similar enough so that in the
following theorem we state a sufficient condition for either of the SDPs to exactly recover $Z^*$ with high probability.
Define
\begin{equation}
	I(\mu,d)\triangleq \mu - d\log \frac{e\mu}{d}.
	\label{eq:I}
\end{equation}     For $\mu>0$ fixed, $I(\mu,d)$ is a strictly
convex, nonnegative function in $d$ which is zero if and only if $d=\mu.$
\begin{theorem}  \label{thm:SDP_gen}
Suppose there exists $\psi_1>0$ and $\psi_2> 0$ with $b+ \psi_1 + \psi_2 < a$ such that
\begin{eqnarray}
I(a, b+\psi_1+\psi_2 ) > 1/\rho_r   \label{eq:con1}  \\
I(b, b+\psi_1) > 1/\rho_r    \label{eq:con2}   \\
I(b, b+ \psi_2) > 1/\rho_{r-1}    \label{eq:con3}   \\
I(b, b+\psi_1+\psi_2) > 1/\rho_r   \label{eq:con4}   
\end{eqnarray}
(with the understanding that \eqref{eq:con2} and \eqref{eq:con3} can be dropped if there is only one cluster (i.e. $r=1$) and
\eqref{eq:con4} can be dropped unless there is only one cluster plus outlier vertices).
Let $\eta^* = C\sqrt{\log n}$ for a sufficiently large constant $C$
and let $\lambda^*=\frac{(b+\psi_1+\psi_2)\log n}{n}.$   If $\widehat{Z}_{\SDP}$ is produced by
either SDP \eqref{eq.SDP_RZ} or SDP \eqref{eq.SDP_RZaug},  then
$\min_{Z^* \in {\cal Z}_n} \Prob\{ \widehat{Z}_{\SDP} = Z^* \} \geq 1- n^{-\Omega(1)}.$
\end{theorem}

We examine two simpler sufficient conditions for recovery,  assuming
we have enough information to implement one of the two SDPs, and we also have a lower bound
$\rho_{\min},$ on the $\rho_k$'s,  but we don't  know how many clusters there are nor whether
there are outlier vertices.    The conditions of Theorem \ref{thm:SDP_gen} are most stringent when
there are two clusters of the smallest possible size $\rho_{\min},$ and in that case we get the tightest result from the
theorem by selecting $\psi_1=\psi_2=\psi,$ yielding the following corollary:
\begin{corollary} \label{cor21}
Let $\psi$ be the solution to $I(a,b+2\psi)=I(b,b+\psi).$   (It satisfies $b  < b+2\psi < a.$)
If $I(b,b+\psi) > 1/\rho_{\min},$  then
$\min_{Z^* \in {\cal Z}_n} \Prob\{ \widehat{Z}_{\SDP} = Z^* \} \geq 1- n^{-\Omega(1)}.$
\end{corollary}
There is no simple expression for $\psi$ in Corollary \ref{cor21}.   If instead we consider
the equation  $I(a,b+2\psi)=I(b,b+2\psi),$ we have the smaller but explicit solution $\psi = \frac{\tau -b}{2},$  
where $\tau=\frac{a-b}{\log(a/b)}.$
Using this $\psi$ in
the test $I(b,b+\psi) > 1/\rho_{\min},$ we obtain the following weaker but more explicit recovery condition, which, nevertheless, is within a \emph{factor of eight} of the necessary condition (see \prettyref{rmk:factor8} below):
\begin{corollary}  \label{cor22}
If $I\left(b, \frac{\tau+b}{2}\right) > 1 / \rho_{\min}$ then
$\min_{Z^* \in {\cal Z}_n} \Prob\{ \widehat{Z}_{SDP} = Z^* \} \geq 1- n^{-\Omega(1)}.$
(If SDP \eqref{eq.SDP_RZaug} is used, it is assumed that $\eta^*$ and $\lambda^*$ are selected as in
Theorem \ref{thm:SDP_gen}, namely, $\eta^*=C\sqrt{\log n}$ and $\lambda^* =  \frac{ \tau \log n}{n},$
where $\tau=\frac{a-b}{\log(a/b)}.$ )
\end{corollary}

\begin{remark}
Let us compare the sufficient condition provided by Corollary \ref{cor22} with
necessary conditions for recovery.   
In the presence of outliers,  $I(b, \tau)>1/\rho_{\min}$ is a necessary
condition as shown in \cite[Theorem 4]{HajekWuXuSDP14}, for otherwise we can swap a vertex in the smallest
cluster with an outlier vertex to increase the number of in-cluster edges.
Also, with at least two clusters,
\begin{equation}  \label{eq.two_cluster_cond}
( \sqrt{a}- \sqrt{b} )^2 >1/\rho_{\min}
\end{equation}
is necessary, because we could have two smallest clusters
of sizes $\rho_{\min} n,$   and even if a genie were to reveal all the other clusters, we would still need
\eqref{eq.two_cluster_cond} to recover the two smallest ones, as  shown by \cite[Theorem 1]{Abbe14}.
By Lemma \ref{lemma.Iab},  $I (b, \tau) \leq  ( \sqrt{a} -\sqrt{b} )^2  \leq  2 I(b, \tau)$;    so
with or without outliers,   $2 I(b, \tau) > 1/\rho_{\min}$ is necessary.
By Lemma \ref{lemma.Ibnd},   $I(b, \tau ) \leq 4 I ( b, \frac{\tau+b}{2} )$. Therefore we conclude that the sufficient condition of Corollary \ref{cor22} is within a factor of four (resp. eight) of the necessary condition in the presence (resp. absence) of outliers.
\label{rmk:factor8}
\end{remark}

\section{Conclusions}

This paper shows that the SDP procedure works for recovering community structure
at the asymptotically optimal threshold in various important settings beyond the case
of two equal-sized clusters  or that of a single cluster and outliers considered in \cite{HajekWuXuSDP14}.   In particular,
SDP relaxations works asymptotically optimally for two unequal clusters (with or without knowing the cluster size), or $r$
equal clusters, or the binary censored block model with
the background graph being Erd\H{o}s-R\'enyi.
These results demonstrate the versatility of SDP relaxation as a simple, general purpose, computationally feasible methodology for community detection.

The picture is less impressive when these cases are combined to have a general case with $r$ clusters of various sizes  plus outliers.
Still, we found that
an SDP procedure can achieve exact recovery even without the knowledge of the cluster sizes;
the sufficient condition for recovery is within a factor of eight of the necessary information-theoretic bound.
An interesting open problem is whether the SDP relaxation can achieve the optimal recovery threshold
in this general case.  The preprint \cite{perry2015semidefinite} addresses this problem, showing that the SDP relaxation
still achieves the optimal threshold for recovering a fixed number of  clusters with unequal  sizes.

\section{Proofs}
\subsection{Proofs for \prettyref{sec:asym}: Binary asymmetric SBM }
\begin{lemma}[{\cite[Lemma 2]{HajekWuXuSDP14}}]\label{lmm:binomialmaxminconcentration}
Let $X \sim \Binom\left( m , \frac{a\log n}{n} \right)$ and $R \sim \Binom\left( m, \frac{b\log n}{n} \right)$ for $m \in \naturals$ and $a,b>0$, where $m =\rho n +o(n) $ for some $\rho>0$ as $n\diverge$. Let $k_n,k_n' \in [m]$ be such that $k_n= \tau \rho \log n+o(\log n)$ and $k'_n= \tau' \rho \log n+o(\log n)$ for some $0 \leq \tau \leq a$ and $\tau' \geq b$.
Then
\begin{align}
\prob{ X \le k_n } &=  n^{-  \rho \left( a - \tau \log \frac{\eexp a}{\tau} +o(1) \right)   } \label{eq:binomupbound1} \\
\prob{ R \ge  k_n' } & =  n^{-  \rho  \left( b - \tau' \log \frac{\eexp b}{\tau'} +o(1) \right)     } \label{eq:binomupbound2}.
\end{align}
\end{lemma}

\newcommand{\ttau}{\alpha}
\begin{lemma}  \label{lmm:general}
Suppose $a, b > 0$, $\ttau \in \IR,$ and either $\rho_1>0$ or $\rho_2>0.$
Let $X$ and $R$ be independent with $X \sim  \Binom(m_1, \frac{a\log n}{n})$   and
$R \sim \Binom(m_2, \frac{b\log n}{n}),$ where $m_1=\rho_1 n+o(n)$ and
$m_2=\rho_2 n +o(n)$  as $n \to \infty$. Let $k \in \naturals$ such that $k= \ttau \log n+o(\log n)$.
If $\ttau \leq a \rho_1 - b\rho_2 ,$
 \begin{align}
 \prob{X-R \leq  k  } = n^{-g(\rho_1, \rho_2, a, b,\ttau) +o(1) }, \label{eq:tail1}
 \end{align}  where
\begin{align*}
g(\rho_1, \rho_2, a, b, \ttau) = \left\{
    \begin{array}{rl}
 a\rho_1 + b\rho_2 - \gamma - \frac{\ttau}{2} \log \frac{(\gamma-\ttau)a\rho_1}{(\gamma+\ttau)b\rho_2}   & \rho_1, \rho_2>0\\
  \rho_2 b + \ttau \log \frac{- \eexp \rho_2 b }{ \ttau } &  \rho_1=0, \rho_2>0 \\
    \rho_1 a - \ttau \log  \frac{\eexp \rho_1 a }{ \ttau } & \rho_1>0, \rho_2=0
    \end{array} \right.,
\end{align*}
with $\gamma=\sqrt{\ttau^2 + 4\rho_1\rho_2ab }.$

Furthermore, for any $m_1, m_2, k \in \naturals$ such that $k \leq (m_1 a  - m_2 b) \log n /n$,
 \begin{align}
 \prob{X-R \leq  k   }  \le n^{-g(m_1/n, m_2/n, a, b, k/\log n) }. \label{eq:tail2}
 \end{align}
\end{lemma}
\begin{proof}
We first prove the upper tail bound in \prettyref{eq:tail2}
using Chernoff's bound. In particular,
\begin{align}
\prob{X-R \leq k} \le \exp \left( -n \ell (k/n) \right),   \label{eq:chernoff}
\end{align}
where $\ell(x) =\sup_{t \ge 0} -t x  - \frac{1}{n} \log \expect{\eexp^{-t(X-R)}}$.
Let $\rho_{1,n}= m_1/n$, $\rho_{2,n}=m_2/n$ and $\ttau_n=k/ \log n$.
By definition,
\begin{align*}
\frac{1}{n} \log \expect{\eexp^{-t(X-R)}} = \rho_{1,n} \log \left(1- \frac{a\log n}{n} (1-\eexp^{-t} ) \right) + \rho_{2,n} \log \left( 1 - \frac{b\log n}{n} (1- \eexp^{t} ) \right).
\end{align*}
Since $-t x -  \frac{1}{n} \log \expect{\eexp^{-t(X-R)}}$ is concave in $t$, it achieves the supremum at $t^\ast$ such that
\begin{align*}
-x  +  \rho_{1,n} \frac{a \eexp^{-t^\ast}\log n /n }{1- \frac{a\log n}{n} (1-\eexp^{-t^\ast} ) } - \rho_{2,n} \frac{b \eexp^{t^\ast} \log n /n  }{1- \frac{a\log n}{n} (1-\eexp^{-t^\ast} ) } =0.
\end{align*}
It suggests that when $x= k/n$, we choose
\begin{align*}
t^\ast = \left \{
   \begin{array}{rl}
 \log \frac{\gamma_n - \ttau_n}{2 \rho_{2,n} b}  & \rho_{1,n}, \rho_{2,n}>0\\ [1em]
\log  \frac{-\ttau_n}{\rho_{2,n} b }  &  \rho_{1,n}=0, \rho_{2,n}>0 \\ [1em]
\log \frac{\rho_{1,n} a }{ \ttau_n } & \rho_{1,n}>0, \rho_{2,n}=0
    \end{array} \right.,
   \end{align*}
    with $\gamma_n= \sqrt{\ttau_n^2+4 \rho_{1,n} \rho_{2,n}  a b}$. Thus,
using the inequality that $\log(1 -x) \le -x$, we have
\begin{align*}
\ell(k/n) & \ge  \left(- \ttau_n t^\ast  + \rho_{1,n} a + \rho_{2,n} b - \rho_{1,n} a  \eexp^{-t^\ast} - \rho_{2,n} b  \eexp^{t^\ast} \right) \frac{\log n}{n} \\
&=g( \rho_{1,n}, \rho_{2,n}, a, b, \ttau_n )\frac{\log n}{n}.
\end{align*}
Then in view of \prettyref{eq:chernoff},
\begin{align*}
\prob{X-R \leq  k   } \le n^{- g(\rho_{1,n} , \; \rho_{2,n}, \; a, \; b, \; \ttau_n) }.
\end{align*}
If $\rho_{1,n}=\rho_1+o(1)$, $\rho_{2,n}=\rho_2+o(1)$, and $\ttau_n=\ttau+o(1)$, then we let
\begin{align*}
t^\ast = \left \{
   \begin{array}{rl}
 \log \frac{\gamma - \ttau}{2 \rho_{2} b}  & \rho_{1}, \rho_{2}>0\\ [1em]
\log  \frac{-\ttau}{\rho_{2} b }  &  \rho_{1}=0, \rho_{2}>0 \\[1em]
\log \frac{\rho_{1} a }{ \ttau } & \rho_{1}>0, \rho_{2}=0
    \end{array} \right.,
   \end{align*}
    with $\gamma= \sqrt{\ttau^2+4 \rho_{1} \rho_{2}  a b}$.
It follows that
\begin{align*}
\ell(k/n) & \ge  \left(- \ttau t^\ast  + \rho_{1} a + \rho_{2} b - \rho_{1} a  \eexp^{-t^\ast} - \rho_{2} b  \eexp^{t^\ast} \right) \frac{\log n}{n} + o(\log n /n)  \\
&=g( \rho_{1}, \rho_{2}, a, b, \ttau)\frac{\log n}{n} + o(\log n/ n).
\end{align*}
and thus the upper tail bound in \prettyref{eq:tail1} holds in view of \prettyref{eq:chernoff}.
Next, we prove the lower tail bound in \prettyref{eq:tail1}.

{\bf Case 1:} $\rho_1, \rho_2>0$.
For any choice of the constant $\ttau'$  with  $\ttau' > |\ttau|,$
$$\{ X-R \leq \ttau \log n \} \supset  \left\{  X \leq  \frac{\ttau' + \ttau }{2}  \log n \right\} \cap
\left \{  R \geq  \frac{\ttau' - \ttau }{2}  \log n \right\}
 $$
 and therefore
 \begin{align*}    
 \prob{ X-R \leq \ttau \log n } \geq \max_{\ttau': \ttau' >|\ttau| }  \prob{ X \leq  \frac{\ttau' + \ttau }{2}  \log n}  \prob{
 R \geq  \frac{\ttau' - \ttau}{2}  \log n}.
\end{align*}
 So, applying \prettyref{lmm:binomialmaxminconcentration}, we get that
 \begin{align*}
  \prob{ X-R \leq \ttau \log n } \geq \max_{\ttau': \ttau' >|\ttau| } n^{-\left[  a\rho_1 - \frac{\ttau'+ \ttau}{2} \log   \frac{2 e \rho_1 a}{\ttau' + \ttau}    +
 b\rho_2 - \frac{\ttau' -  \ttau}{2} \log  \frac{2 e \rho_2 b}{\ttau' - \ttau}   +o(1)  \right]}. 
 \end{align*}
Setting $\ttau'=  \sqrt{\ttau^2+4 \rho_1 \rho_2 a b}$ in the last displayed equation yields
 \begin{align*}
 \prob{X-R \leq  \ttau \log n  } \ge n^{-\left[  a\rho_1 + b\rho_2 - \gamma - \frac{\ttau}{2} \log  \frac{(\gamma-\ttau)a\rho_1}{(\gamma+\ttau)b\rho_2}   +o(1)  \right]}.
 \end{align*}

 {\bf Case 2:} $\rho_1=0, \rho_2>0$. We have
 \begin{align*}
\{ X-R \leq \ttau \log n \} \supset  \left\{  X \leq 2 m_1 a \log n/n \right\} \cap
\left \{  R \geq -\ttau \log n + 2 m_1 a \log n /n \right\}
 \end{align*}
 and therefore
 \begin{align*}
 \prob{ X-R \leq \ttau \log n } &  \geq  \prob{ X \leq 2 m_1 a \log n/n}  \prob{
  R \geq -\ttau \log n + 2 m_1 a \log n /n}  \\
  & \geq \frac{1}{2}n^{  - \rho_2 b + \ttau \log  - \frac{-\eexp \rho_2 b }{ \ttau }  + o(1) },
\end{align*}
where the last inequality follows because by Markov's inequality, $  \prob{ X \leq 2 m_1 a \log n/n} \ge 1/2$, and in view of \prettyref{lmm:binomialmaxminconcentration} with $m_1 \log n /n=o(\log n)$,
\begin{align*}
\prob{ R \geq -\ttau \log n + 2 m_1 a \log n /n} \le n^{  - \rho_2 b -  \ttau \log \frac{-\eexp \rho_2 b }{ \ttau } + o(1) }.
\end{align*}

{\bf Case 3:} $\rho_1>0, \rho_2=0$. We have
 \begin{align*}
\{ X-R \leq \ttau \log n \} \supset  \left\{  X \leq \ttau \log n /n  \right\}
 \end{align*}
 and therefore
 \begin{align*}
 \prob{ X-R \leq \ttau \log n } \geq  \prob{ X \leq  \ttau  \log n/n}  \geq n^{  - \rho_1 a + \ttau \log  \frac{\eexp \rho_1 a }{ \ttau }  + o(1) },
\end{align*}
where the last inequality follows from \prettyref{lmm:binomialmaxminconcentration}.
\end{proof}

The following lemma provides a deterministic sufficient condition for the success of \SDP \prettyref{eq:SBMconvex_unbalanced} in the case $a>b$.
\begin{lemma}\label{lmm:SBMKKT}
Suppose there exist $D^\ast=\diag{d^\ast_i}$ and $ \lambda^\ast \in \reals$ such that $S^* \triangleq D^\ast-A + \lambda^\ast \allones$ satisfies $S^\ast \succeq 0$, $\lambda_2(S^\ast)>0$ and
\begin{align}
S^\ast \sigma^\ast = 0 . \label{eq:SBMKKT_unbalanced}
\end{align}
Then $\widehat{Y}_{\SDP}=Y^\ast$ is the unique solution to \prettyref{eq:SBMconvex_unbalanced}.
\end{lemma}
\begin{proof}
The Lagrangian function is given by
\begin{align*}
L(Y, S, D, \lambda) = \langle A, Y \rangle + \langle S, Y \rangle - \langle D, Y -\identity \rangle - \lambda \left( \Iprod{\allones}{Y} - (2K-n)^2 \right),
\end{align*}
where the Lagrangian multipliers are denoted by $S \succeq 0$, $D=\diag{d_i}$, and $\lambda \in \reals$.
Then for any $Y$ satisfying the constraints in \prettyref{eq:SBMconvex_unbalanced},
\begin{align*}
 \Iprod{A}{Y } & \overset{(a)}{\le} L(Y, S^\ast, D^\ast, \lambda^\ast) = \Iprod{D^\ast}{I} + \lambda^\ast (2K-n)^2
=\Iprod{D^\ast}{Y^\ast} + \lambda^\ast (2K-n)^2  \\
& =\Iprod{A+S^\ast-\lambda^\ast \allones}{Y^\ast} + \lambda^\ast (2K-n)^2 \overset{(b)}=\Iprod{A}{Y^\ast},
\end{align*}
where $(a)$ holds because $\Iprod{S^\ast}{Y} \ge 0$; $(b)$ holds because $\Iprod{Y^\ast}{S^\ast}=(\sigma^\ast)^\top S^\ast \sigma^\ast =0$ by \prettyref{eq:SBMKKT_unbalanced}.
Hence, $Y^\ast$ is an optimal solution. It remains to establish its uniqueness. To this end, suppose $\tY$ is
an optimal solution. Then,
\begin{align*}
\Iprod{S^\ast}{\tY}=\Iprod{D^\ast-A+ \lambda^\ast \allones}{\tY} \overset{(a)}{=} \Iprod{D^\ast-A+ \lambda^\ast \allones}{Y^\ast}{=}\Iprod{S^\ast}{Y^\ast} =0.
\end{align*}
where $(a)$ holds because $\Iprod{\allones}{\tY}=\Iprod{\allones}{Y^\ast}$, $\Iprod{A}{\tY}=\Iprod{A}{Y^\ast}$, and $\tY_{ii}=Y^*_{ii}=1$ for all $i \in [n]$.
In view of \prettyref{eq:SBMKKT_unbalanced}, since $\tY \succeq 0$, $S^\ast \succeq 0$ with $\lambda_2(S^*)>0$, $\tY$ must be a multiple of $Y^*=\sigma^\ast (\sigma^\ast)^\top$.
Because $\tY_{ii}=1$ for all $i \in [n]$, $\tY=Y^\ast$.
\end{proof}

\begin{proof}[Proof of \prettyref{thm:SBMSharp_unbalanced}]
Let $D^\ast=\diag{d^\ast_i}$ with
\begin{equation}
d^\ast_i  = \sum_{j=1}^n A_{ij} \sigma^\ast _i \sigma^* _j	- \lambda^\ast (2K -n )  \sigma_i^\ast
	\label{eq:di-SBM}
\end{equation}
and choose $\lambda^* = \tau \log n/ n$, where $\tau= \frac{a-b}{\log a - \log b}$.
It suffices to show that $S^* = D^\ast-A + \lambda^\ast \allones$ satisfies the conditions in \prettyref{lmm:SBMKKT} with high probability.

By definition, $d^\ast_i \sigma_i^\ast = \sum_{j} A_{ij} \sigma^\ast _j - \lambda^\ast (2K-n)$ for all $i \in [n]$, \ie, $D^\ast \sigma^\ast =A \sigma^\ast - \lambda^\ast (2K-n) \mathbf{1} $. Since $ \allones \sigma^\ast=(2K-n) \mathbf{1}$,
it follows that the desired \prettyref{eq:SBMKKT_unbalanced} holds, that is, $S^*\sigma^* = 0$. It remains to verify that $S^\ast \succeq 0$ and $\lambda_2(S^\ast)>0$ with high probability, which amounts to showing that
\begin{equation}
\prob{\inf_{x \Perp \sigma^\ast, \|x\|_2=1} x^\top S^\ast x  > 0}  \ge 1- n^{-\Omega(1) }.	
	\label{eq:lambda2_unbalanced}
\end{equation}
Note that $\expect{A}= \frac{p-q}{2} Y^\ast + \frac{p+q}{2} \allones - p \identity$ and $Y^\ast= \sigma^\ast (\sigma^\ast)^\top$. Thus for any $x$ such that $x \Perp \sigma^\ast$ and $\|x\|_2=1$,
\begin{align}
x^\top S^\ast x &= x^\top D^\ast x- x^\top \expect{A} x  +  \lambda^\ast x^\top \allones x - x^\top  \left( A - \expect{A} \right) x \nonumber  \\
&  = x^\top D^\ast  x  - \frac{p-q}{2} x^\top Y^\ast x + \left( \lambda^\ast- \frac{p+q}{2} \right) x^\top \allones x +p  -
x^\top  \left( A- \expect{A} \right) x  \nonumber \\
& \overset{(a)}{=} x^\top D^\ast  x + \left( \lambda^\ast- \frac{p+q}{2} \right) x^\top \allones x  +p - x^\top  \left( A- \expect{A} \right) x.
 \label{eq:SBMPSDCheck_unbalanced}
\end{align}
where $(a)$ holds because $\iprod{x}{\sigma^\ast}=0$.
It follows from \prettyref{eq:SBMPSDCheck_unbalanced} that for any  $x \Perp \sigma^\ast$ and $\|x\|_2=1$,
$x^TS^*x=t_1(x) + t_2(x)$ where
\begin{align}
t_1(x) = & ~ 	x^\top D^* x + \left( \lambda^* - \frac{p+q}{2}\right)  x^\top \allones x,  \label{eq:t1}\\
t_2(x)= & ~ 	p -x^\top (A-\expect{A} ) x. \label{eq:t2}
\end{align}

Observe that
\begin{eqnarray}
\inf_{x\perp \sigma^* ,\|x\|_2 =1} x^\top S^*x \geq  \inf_{x\perp\sigma^* ,\|x\|_2 =1} t_1(x) +  \inf_{x\perp \sigma^* ,\|x\|_2 =1} t_2(x).  \label{eq:PSD_check_unbalanced}
\end{eqnarray}
Now  $ \inf_{x\perp\sigma^* ,\|x\|_2 =1} t_2(x)  \geq p  - \| A -\Expect[A] \|$.
In view of \cite[Theorem 5]{HajekWuXuSDP14}, with high probability $\|  A - \expect{A}\| \leq c' \sqrt{\log n} $
 for a positive constant $c'$ depending only on $a$ and thus $ \inf_{x\perp\sigma^* ,\|x\|_2 =1} t_2(x) \geq p-c' \sqrt{\log n}$.

We next bound  $\inf_{x\perp\sigma^* ,\|x\|_2 =1} t_1(x) $ from the below. Consider the specific vector $\check{x}$ that maximizes $x^\top \allones x$ subject to the unit norm constraint and  $\Iprod{x}{\sigma^*}=0.$
It has coordinates $\sqrt{\frac{n-K}{nK}}$  for the $K$ vertices of the first cluster and coordinates $\sqrt{\frac{K}{n(n-K)}}$ for the
$n-K$ vertices of the other cluster.  Let $E_2=\mbox{span}(\sigma^*,\check{x});$  $E_2$ is the set of vectors that are constant over each cluster.
Then
\begin{align*}
 \inf_{x\perp\sigma^* ,\|x\|_2 =1} t_1(x)     =  \inf_{ \beta\in [0,1], \|x\|_2=1,x\perp E_2}  t_1\left( \beta \check{x}  + \sqrt{1-\beta^2}x   \right)  .
\end{align*}
Notice that for any vector $x$ with $x  \perp E_2,$   $\allones x=0.$ It follows that
\begin{eqnarray*}
 t_1\left( \beta \check{x}  + \sqrt{1-\beta^2}x   \right)  = \beta^2 t_1(\check{x}) + 2\beta\sqrt{1-\beta^2}x^\top D^*\check{x} + (1-\beta^2)x^\top D^*x.
\end{eqnarray*}
Therefore,
\begin{align*}
 \inf_{x\perp\sigma^* ,\|x\|_2 =1} t_1(x)     \ge   \inf_{ \beta\in [0,1]}  \left(  \beta^2 t_1(\check{x}) + 2\beta\sqrt{1-\beta^2}   \inf_{\|x\|_2=1,x\perp E_2} x^\top D^*\check{x} + (1-\beta^2)  \inf_{\|x\|_2=1,x\perp E_2} x^\top D^*x \right).
\end{align*}
We bound the three terms in the parenthesis separately in the sequel.

{\bf Lower bound on $ t_1(\check{x})$}:
Notice that
$\check{x}^\top \allones \check{x}=4K(n-K)/n$ and thus
\begin{eqnarray*}
\left(\lambda^* -\frac{p+q}{2}\right) \check{x}^\top \allones \check{x} & = &   \left(\tau - \frac{a+b}{2}\right) 4K(n-K)  \log n/ n^2 \\
& = &    \left(\tau -  a + \tau - b \right) 2 K(n-K)  \log n/n^2.
\end{eqnarray*}
If $\sigma^*_i=+1$ then
\begin{align*}
\expect{d^*_i}  \triangleq \bar{d}_+ = \left\{  K (a-\tau) + (n-K)(\tau-b) - a \right \} \log n/n.
\end{align*}
If $\sigma^*_i=-1$ then
\begin{align*}
\expect{d^*_i}  \triangleq \bar{d}_- = \left\{  (n-K)(a-\tau) + K(\tau-b)  - a \right \} \log n/n .
\end{align*}
Therefore,
\begin{eqnarray*}
\expect{\check{x}^\top D^* \check{x}}  & = & (n-K) \bar{d}_+ / n  +  K \bar{d}_- /n  \\
&=&  \left\{ 2K (n-K) (a-\tau)  + \left(K^2 + (n-K)^2 \right)(\tau-b)   - n a \right\}  \log n/n^2.
\end{eqnarray*}
Since $\check{x}^\top D^\ast \check{x} =  \Iprod{A}{B} -\lambda^\ast (2K -n) \sum_{i=1}^n \check x_i^2 \sigma_i^\ast$, where $B_{ij} = \sigma_i \sigma_j \check x_i^2$,
it follows that $\check{x}^\top D\check{x}$ is Lipschitz continuous in $A$ with Lipschitz constant $\fnorm{B} = \sqrt{(1-\rho)^2/{\rho}+\rho^2/(1-\rho)} + o(1)$.
Moreover, $A_{ij}$ is $[0,1]$-valued. It follows from the Talagrand's concentration inequality for Lipschitz convex functions (see, \eg, \cite[Theorem 2.1.13]{tao.rmt}) that
for any $c>0$, there exists $c'>0$ only depending on $\rho$, such that
\begin{align*}
\prob{  \check{x}^\top D^* \check{x}  - \expect{\check{x}^\top D^* \check{x} } \ge - c'  \sqrt{ \log n} }  \ge 1 - n^{-c}.
\end{align*}
Hence, with probability at least $1- n^{-c}$,
\begin{equation}   \label{eq.Jterm}
t_1( \check{x} ) \ge  \expect{\check{x}^\top D^*\check{x}} + \left(\lambda^* -\frac{p+q}{2}\right)\check{x}^\top \allones \check{x}  - c' \sqrt{\log n}
= (\tau -b)\log n  - a \log n /n - c' \sqrt{\log n} .
\end{equation}

{\bf Lower bound on $\inf_{\|x\|_2 =1, x\perp E_2}  x^\top D^*\check{x}$}:
Note that  $E[D^*]\check{x}\in E_2.$   So for any vector $x$ with $x  \perp E_2,$   $x^\top E[D^*]\check{x}=0$. Hence,
\begin{align}
 \inf_{\|x\|_2 =1, x\perp E_2}  x^\top D^*\check{x} =  \inf_{\|x\|_2 =1, x\perp E_2}  x^\top (D^*-\expect{D^*})\check{x} \geq - \|  (D^*-\expect{D^*})\check{x} \|. \label{eq:lowboundterm2}
\end{align}
 Notice that
\begin{align*}
\|  ( D- \expect{D} ) \check{x}  \|_2^2 =
\sum_i \left( \sum_{j=1}^n ( A_{ij}- \expect{A_{ij}}  ) \sigma_i^\ast \sigma_j^\ast  \check{x} _i   \right)^2 = \sum_i \check{x} _i^2
\left( \sum_{j=1}^n ( A_{ij}- \expect{A_{ij}}  ) \sigma_j^\ast \right)^2 .
\end{align*}
 Therefore
 \begin{align*}
  \expect{ \|(D-E[D]) \check{x} \|}  \le \sqrt{ \expect{\| (D- \expect{D} ) \check{x}  \|_2^2} } = \sqrt{\sum_i \check{x} _i^2 \sum_{j=1}^n \var[A_{ij}] } \le \sqrt{a \log n},
  \end{align*}
One can check that $\|(D-\expect{D}  ) \check{x}  \|$ is convex and Lipschitz continuous in $A$ with Lipschitz constant bounded by $\max \{ \sqrt{ \frac{1-\rho}{\rho}}, \sqrt{ \frac{\rho}{1-\rho}} \}.$ In particular, for any given $A, A'$,
let $D, D'$ denote the corresponding diagonal matrix, then
\begin{align*}
\big | \|(D-\expect{D}  ) \check{x}  \| - \|(D'-\expect{D'} ) \check{x}  \|  \big|  & \le  \|  (D- D') \check{x} \|   = \sqrt{ \sum_i \check{x} _i^2   (\sum_j (A_{ij}- A'_{ij} ) \sigma^\ast_j )^2 } \\
& \le \sqrt{ \sum_i (\sum_j (A_{ij}- A'_{ij} ) \sigma^\ast_j )^2 } \max \left \{ \sqrt{ \frac{n-K}{nK}}, \sqrt{ \frac{K}{n(n-K)} } \right \} \\
& \le \fnorm{A-A'} \max \left \{ \sqrt{ \frac{1-\rho}{\rho}}, \sqrt{ \frac{\rho}{1-\rho}} \right \},
\end{align*}
where the last inequality follows from the Cauchy-Schwartz inequality.
It follows from the Talagrand's concentration inequality for Lipschitz convex functions that for any $c>0$, there exists $c'>0$ such that
\begin{align*}
\prob{ \|(D-E[D]) \check{x} \|  - \expect{\|(D-E[D]) \check{x} \|}  \le c'  \sqrt{ \log n} }  \ge 1 - n^{-c}.
\end{align*}
Hence, with probability at least $1-n^{-c}$,  $ \|  (D^*-\expect{D^*} )\check{x} \| \le  c' \sqrt{\log n}$ for some universal constant $c'$ and
$\inf_{\|x\|_2 =1, x\perp E_2}  x^\top D^*\check{x} \ge - c'\sqrt{\log n}$ in view of \prettyref{eq:lowboundterm2}.

{\bf Lower bound on $\inf_{ \|x\|_2=1,x\perp E_2 } x^\top D^*x$}:
 Notice that for $\|x\|_2=1,$  $x^\top D^* x \geq \min_i d^\ast_i$, so it suffices to bound $\min_i d^\ast_i$ from the below.
For $i \in C_1$, $A_{ij}\sigma_i \sigma_j$ is equal in distribution to $X-R$, where
$X \sim \Binom(K -1 ,\frac{a\log n}{n})$ and $R \sim \Binom(n-K,\frac{b\log n}{n})$.
It follows from \prettyref{lmm:general} that
\begin{align*}
\prob{\sum_{j} A_{ij}\sigma_i \sigma_j \le - \tau (n-2K) \log n/n + \frac{\log n}{\log \log n} }   \le n^{-\eta(\rho, a, b) +o(1)}.
\end{align*}
For $i \in C_2$, $A_{ij}\sigma_i \sigma_j$ is equal in distribution to $X-R$, where
$X \sim \Binom(n-K-1 ,\frac{a\log n}{n})$ and $R \sim \Binom( K,\frac{b\log n}{n})$.
It follows from \prettyref{lmm:general} that
\begin{align*}
\prob{\sum_{j} A_{ij}\sigma_i \sigma_j \le  \tau (n-2K) \log n/ n + \frac{\log n}{\log \log n}}  \le n^{-\eta(\rho, a, b) +o(1)}.
\end{align*}
It follows from the definition of $d^\ast_i$ that
\begin{align*}
\prob{d^\ast_i \ge \frac{\log n}{\log \log n} } \ge 1- n^{-\eta(\rho, a, b) +o(1)}, \forall i
\end{align*}
Applying the union bound, we get that
\begin{align*}
\prob{ \min_{i\in[n]} d^\ast_i  \ge \frac{\log n}{\log \log n}  } \ge 1 - n^{1- \eta(\rho,a ,b)+o(1)} \ge 1-n^{-\Omega(1) },
\end{align*}
where the last inequality follows from the assumption that $\eta(\rho,a ,b)>1$.

Combing all the three lower bounds together, we get that with high probability,
\begin{eqnarray*}
 \inf_{x\perp\sigma^* ,\|x\|_2 =1} t_1(x)  &  \geq
 & \inf_{\beta \in [0,1]} \beta^2  (\tau-b)\log n   - 2 c' \beta\sqrt{1-\beta^2}\sqrt{\log n } + (1-\beta)^2\frac{\log n}{\log\log n} - c' \sqrt{ \log n}  \\
& \geq & \frac{1}{2} \min \left \{ (\tau-b) \log n , \frac{\log n }{\log\log n} \right \}  -  3c'\sqrt{\log n},
\end{eqnarray*}
Notice that we have shown that with high probability $ \inf_{x\perp\sigma^* ,\|x\|_2 =1} t_2(x) \geq p-c' \sqrt{\log n}$.
It follows from  \prettyref{eq:PSD_check_unbalanced} that with high probability,
\begin{align*}
\inf_{x\perp \sigma^* ,\|x\|_2 =1} x^\top S^*x \geq  \frac{1}{2} \min \left \{ (\tau-b) \log n , \frac{\log n }{\log\log n} \right \}  -  4c'\sqrt{\log n} +p .
\end{align*}
Notice  that $a>b>0$ and thus $\tau>b$. Therefore, the desired \prettyref{eq:lambda2_unbalanced} holds
and the theorem follows from \prettyref{lmm:SBMKKT}.
\end{proof}

\begin{proof}[Proof of \prettyref{thm:SBMSharpConverse_unbalanced}]

Since the prior distribution of $\sigma^\ast$ is uniform over $\{ \sigma \in \{ \pm 1\}^n: \sigma^\top \mathbf{1} =2 K-n \}$,
the \ML estimator minimizes the error probability among all estimators and thus we only need to find when the \ML estimator fails.
Let $C_1^\ast, C_2^\ast$ denote the true cluster $1$ and $2$, respectively.
Let $e(i, T) \triangleq \sum_{j \in T} A_{ij}$ for a set $T$.  Recall that $\tau=\frac{a-b}{\log a - \log b}.$ Let $F_1$ denote the event that
$\min_{i \in C_1^*} ( e(i, C_1^\ast) -e (i, C_2^\ast) ) \le  - \tau (1-2\rho) \log n -2 $ and $F_2$ denote the
event  that $\min_{i \in C_2^*} ( e(i, C_2^\ast) -e (i, C_1^\ast) ) \le \tau (1-2\rho) \log n-2 $.
Notice that $F_1 \cap F_2$ implies the existence of $i \in C_1^\ast $ and $j \in C_2^\ast $, such that the set $(C_1^*\backslash\{i\}\cup\{j\}, C_2^* \backslash \{j\} \cup \{i\} )$
achieves a strictly higher likelihood than $(C_1^*, C_2^\ast)$.
Hence $\prob{\text{\ML fails} } \ge \prob{F_1 \cap F_2}$. Next we bound $\prob{F_1}$ and $\prob{F_2}$ from below.

By symmetry, we can condition on $C_1^*$ being the first $K$ vertices.
Let $T$ denote the set of first $ \lfloor \frac{n}{\log^2 n} \rfloor$ vertex.
Then
\begin{equation}
\min_{i \in C_1^*} ( e(i, C_1^\ast) -e (i, C_2^\ast) )   \leq \min_{i \in T} ( e(i, C_1^\ast) -e (i, C_2^\ast) )  \leq \min_{i \in T} ( e(i, C_1^*\backslash T) -e (i, C_2^\ast)  )  + \max_{i \in T} e(i, T).	
	\label{eq:wa_unbalanced}
\end{equation}
Let $E_1,E_2$ denote the event that $\max_{i \in T} e(i,T) \le \frac{\log n}{\log \log n} -2$, $\min_{i \in T} ( e(i, C_1^*\backslash T)  -e (i, C_2^\ast )  ) \le
- \tau (1-2\rho) \log n - \frac{\log n}{\log \log n} $, respectively.
In view of \prettyref{eq:wa_unbalanced}, we have $F_1 \supset E_1 \cap E_2$ and hence it boils down to proving that $\prob{E_i} \to 1$ for $i=1,2$.

For $i \in T$, $e(i, T) \sim \Binom(|T|, a \log n /n )$ . In view of the following Chernoff bound for binomial distributions \cite[Theorem 4.4]{Mitzenmacher05}: For $r \ge 1$ and $X \sim \Binom(n,p)$, $\prob{X \ge r np }
\le ( \eexp/r)^{rnp},$
we have
\begin{align*}
\prob{ e(i, T) \ge \frac{\log n}{\log \log n} -2} \le   \left( \frac{\log^2 n}{a \eexp \log \log n } \right) ^{- \log  n/ \log \log n+2} = n^{-2+o(1)}.
\end{align*}
Applying the union bound yields
\begin{align*}
\prob{E_1} \ge 1 - \sum_{i \in T} \prob{ e(i, T) \ge \frac{\log n}{\log \log n} -2} \ge
1 - n^{-1+o(1)}.
\end{align*}
Moreover,
\begin{align*}
\prob{E_2} & \overset{(a)}{=} 1- \prod_{i\in T} \prob{ e(i, C_1^*\backslash T) - e( i, C_2^\ast)  >  - \tau (1-2\rho) \log n  - \frac{\log n}{\log \log n} } \\
&  \overset{(b)}{\ge} 1- \left( 1- n^{-  \eta( \rho, a, b ) + o(1)  } \right)^{|T|}
\overset{(c)}{\ge} 1- \exp \left( - n^{1- \eta(\rho, a, b)  + o(1) } \right) \overset{(d)}{\to} 1,
\end{align*}
where $(a)$ holds because $\{e(i, C^*\backslash T)\}_{ i \in T}$ are mutually independent; $(b)$ follows by applying \prettyref{lmm:general}
and noticing that $g(\rho, 1-\rho, a, b, -\tau (1-2\rho) ) = \eta(\rho, a, b)$;
$(c$) is due to $1+x \le e^x$ for all $x \in \reals$;
$(d$) follows from the assumption that $\eta(\rho, a, b) < 1$. Thus $\prob{F_1} \to 1$.
Using the same argument as above, we can show $\prob{F_2} \to 1$. Thus the theorem follows.
\end{proof}

\begin{proof}[Proof of \prettyref{thm:SDP2}]
Suppose that the true clusters are $C^\ast_1$ and $C^\ast_2$ of cardinality $K_n$ and $n-K_n$, respectively.
One can easily check that \prettyref{lmm:SBMKKT} still holds with $\lambda^\ast= \tau \log n/n$, where $\tau= \frac{a-b}{\log a- \log b}$.
Choose the same $d_i^\ast$ in \prettyref{eq:di-SBM} as in the proof of \prettyref{thm:SBMSharp_unbalanced}. It
suffices to show for any $0 \le K_n \le n$,
\begin{equation}
\prob{\inf_{x \Perp \sigma^\ast, \|x\|_2=1} x^\top S^\ast x  > 0} \geq 1- n^{-\Omega(1) }.	\label{eq:PSDcheck_unknown}
\end{equation}

First, consider the case $K_n=0$ or $n$ where $Y^*=\allones$ and the graph is simply $\calG(n,p)$. Then for $i \in [n]$, $\sum_{j} A_{ij}\sigma^\ast_i \sigma^\ast_j \sim \Binom(n-1 ,\frac{a\log n}{n})$.
Recall that $\tau= \frac{a-b}{\log a - \log b}$ and notice that in this case, $d_i^\ast=\sum_{j} A_{ij}\sigma^\ast_i \sigma^\ast_j  - \tau \log n$.
It follows from \prettyref{lmm:binomialmaxminconcentration} that
\begin{align*}
 \prob{d^\ast_i \le \frac{\log n}{\log \log n}} =  \prob{\sum_{j} A_{ij}\sigma^\ast_i \sigma^\ast_j \le  \tau \log n + \frac{\log n}{\log \log n} }
\le  n^{-\eta(0, a, b)+o(1)} \le  n^{-\eta(1/2, a, b)+o(1)},
\end{align*}
where $\eta(0, a, b)= a- \tau \log (\eexp a/\tau)$ and the last inequality follows from \prettyref{lmm:eta_behavior} in \prettyref{app:eta}.
By the union bound,
\begin{align*}
\prob{ \min_{i \in [n]} d^\ast_i \ge \frac{\log n}{\log \log n} }  \geq 1-n^{ 1-\eta(1/2, a, b) }  \ge 1-n^{-\Omega(1) },
\end{align*}
where the last inequality holds because $\eta(1/2,a,b) = \frac{1}{2} (\sqrt{a}-\sqrt{b})^2>1$ by assumption.
Moreover, since $\sigma^*=\pm \ones$,
any $x $ such that $x \perp \sigma^\ast$ satisfies $x^\top \allones x=0$.
It follows from \prettyref{eq:SBMPSDCheck_unbalanced} that
\begin{align*}
x^\top S^\ast x = x^\top D^\ast x + p - x^\top (A-\expect{A})x \ge \min_{i \in [n]} d_i + p - \|A-\expect{A}\|.
\end{align*}
By \cite[Theorem 5]{HajekWuXuSDP14}, $\Prob[\|  A - \expect{A}\| \leq c' \sqrt{\log n}] \geq 1 -o(1)$
 for a positive constant $c'$ depending only on $a$ and thus the desired \prettyref{eq:PSDcheck_unknown} follows.

Next, we consider the case $1 \le K_n \le n-1$.
For $i \in C_1$, $\sum_{j} A_{ij}\sigma_i \sigma_j$ is stochastically larger than $X-R-1$, where
$X \sim \Binom(K_n ,\frac{a\log n}{n})$ and $R \sim \Binom(n-K_n,\frac{b\log n}{n})$.
Let $\rho_n=\frac{K_n}{n} \in (0,1)$ and
$t_n= \tau (1-2\rho_n) \log n - \frac{\log n}{\log \log n} - 1$.
Applying the non-asymptotic upper bound in \prettyref{lmm:general} yields
\begin{align*}
\prob{\sum_{j} A_{ij}\sigma_i \sigma_j \le - \lambda^\ast (n-2K_n)  + \frac{\log n}{\log \log n} }   \le \prob{X-R \le - t_n}   \le n^{-g(\rho_n, 1-\rho_n, a, b, -\frac{t_n}{\log n})}.
\end{align*}
We proceed to show that $g(\rho_n, 1-\rho_n, a, b, -\frac{t_n}{\log n}) \geq \eta(1/2, a, b)+o(1)$.
First note that 
$$
\frac{\partial g}{\partial t}(\rho_1,\rho_2,a,b,t) = - \frac{1}{2} \log \frac{a \rho_1 (\sqrt{4 a b \rho_1 \rho_2+t ^2}-t )}{b \rho_2 (\sqrt{4 a b \rho_1 \rho_2+t ^2}+t )}
$$ 
and $-\frac{t_n}{\log n} = \tau (\rho_n-\bar \rho_n)-\epsilon_n$, where $\epsilon_n= \frac{1}{\log \log n}+\frac{1}{\log n}$ and $\bar \rho_n \triangleq 1- \rho_n$. Furthermore, for any fixed $a,b>0$,
\begin{equation}
\sup_{0<\rho<1} \sup_{0 \le \delta \leq \epsilon_n}\Big|\frac{\partial g}{\partial \tau}(\rho,\bar \rho,a,b,\tau(\rho-\bar \rho) - \delta) \Big|	\leq F(a,b),
	\label{eq:pg}
\end{equation}
for some function $F(a, b)$ independent of $n$ and $\bar \rho \triangleq 1- \rho$.
To see this, let $t=\tau (\rho-\bar \rho) - \delta$, where $0 \leq \delta\leq  \epsilon_n$.
First consider the case of $t < 0$. Then $\rho \leq \frac{1}{2} + o(1) \leq 2/3$. Hence $1/3 \leq \bar \rho \leq 1$. Then
\begin{equation}
	\frac{\rho (\sqrt{4 a b \rho \bar{\rho}+\tau ^2}-\tau )}{\bar{\rho} (\sqrt{4 a b \rho \bar{\rho}+\tau ^2}+\tau )}=\frac{(\sqrt{a b + (\lambda^2-ab) (\bar \rho - \rho)^2 + \delta^2+2\delta\lambda (\bar \rho-\rho)} + (\bar \rho-\rho) \lambda + \delta)^2}{4 a b \bar{\rho}^2}.
	\label{eq:aaa}
\end{equation}
Since $\sqrt{ab} < \tau < \frac{a+b}{2}$ whenever $a \neq b$ and $\bar \rho - \rho \in [-1,1]$, both the numerator and denominator in \prettyref{eq:aaa} are bounded away from zero and infinity uniformly in $\rho$. The case of $t>0$ follows analogously.
Therefore
\begin{align}
	  g\Big(\rho_n, 1-\rho_n, a, b, -\frac{t_n}{\log n}\Big)	
\geq & ~ 	g\Big(\rho_n, 1-\rho_n, a, b,  - \tau (1-2 \rho_n )\Big) - F(a,b) \epsilon_n	\label{eq:g1}\\
= & ~ 	 \eta\Big(\rho_n, a, b\Big) - F(a,b) \epsilon_n	\label{eq:g2}\\
\geq & ~ 	\eta\Big(\frac{1}{2}, a, b\Big) - F(a,b) \epsilon_n	\label{eq:g3}
\end{align}
where \prettyref{eq:g1} is due to \prettyref{eq:pg}, \prettyref{eq:g2} is by definition of $\eta$, and \prettyref{eq:g3} follows from \prettyref{lmm:eta_behavior}.

Similarly, for $i \in C_2$, $A_{ij}\sigma_i \sigma_j$ is stochastically larger than $X-R-1$, where
$X \sim \Binom(n-K_n,\frac{a\log n}{n})$ and $R \sim \Binom( K_n,\frac{b\log n}{n})$.
Let $k'_n= \tau (1-2\rho_n) \log n + \frac{\log n}{\log \log n} +1$.
It follows from \prettyref{lmm:general} that
\begin{align*}
\prob{\sum_{j} A_{ij}\sigma_i \sigma_j \le  \lambda^\ast (n-2K_n) \log n/ n + \frac{\log n}{\log \log n} }
  \le &~\prob{\sum_{j} A_{ij}\sigma_i \sigma_j \le k'_n }  \\
  \le  &~ n^{-g(1-\rho_n, \rho_n, a, b, k'_n/\log n)} \leq n^{-\eta(1/2, a, b)+o(1)},
\end{align*}
where the last inequality follows from the same steps as in \prettyref{eq:g1} -- \prettyref{eq:g3}.
It follows from the definition of $d^\ast_i$ that
\begin{align}
\min_{i\in[n]} \prob{d^\ast_i \ge \frac{\log n}{\log \log n}}
\ge 1-n^{-\eta(1/2,a, b)+o(1)}.
\label{eq:dii}
\end{align}
Applying the union bound gives
\begin{align*}
\prob{ \min_{i \in [n] } d^\ast_i \ge \frac{\log n}{\log \log n} } \geq 1-n^{ 1-\eta(1/2, a, b)+o(1) }  \geq 1-n^{-\Omega(1) },
\end{align*}
where the last inequality follows from the assumption that $\eta(1/2,a,b)= \frac{1}{2} (\sqrt{a}-\sqrt{b})^2>1$.
Furthermore one can verify that $\inf_{x \perp \sigma^\ast} t_2(x) \ge p - O(\sqrt{\log n})$ with high probability, where the functions $t_1$ and $t_2$ are defined in \prettyref{eq:t1}--\prettyref{eq:t2}.
We divide the remaining analysis into the  two cases:

{\bf Case 1:} $K_n \le n/\sqrt{\log n}$ or $n-K_n \le n /\sqrt{ \log n}$. Notice that $\tau \le \frac{a+b}{2}$ and recall the definition of $\check{x}$
in the proof of  \prettyref{thm:SBMSharp_unbalanced}. Then
\begin{align*}
\left(\frac{p+q}{2}  - \lambda^\ast \right ) \check{x}^\top \allones \check{x} = 2(a+b - 2\tau) K_n (n-K_n) \log n /n^2 \le 2(a+b -2\tau) \sqrt{\log n}.
\end{align*}
It follows that with high probability,
\begin{align*}
\inf_{x \perp \sigma^\ast, \|x\|_2=1}  t_1(x) \ge \min_i d^\ast_i  - \left(\frac{p+q}{2}  - \lambda^\ast \right ) \check{x}^\top \allones \check{x}   \ge  \frac{\log n}{\log \log n} - O(\sqrt{\log n}).
\end{align*}
Thus the desired \prettyref{eq:PSDcheck_unknown} follows by the same argument used in the proof of  \prettyref{thm:SBMSharp_unbalanced}.

{\bf Case 2:} $K_n \ge n /\sqrt{\log n}$ and $n-K_n \ge n/\sqrt{ \log n}$. In this case, $\rho_n \ge 1/\sqrt{\log n}$.
The proof is exactly the same as in the proof of \prettyref{thm:SBMSharp_unbalanced} except that
$\check{x}^\top D \check{x}$ and $\| (D- \expect{D} ) \check{x} \| $ are Lipschitz continuous in $A$ with Lipschitz constants
bounded by $O(\log^{1/4} n)$. Thus for any constant $c$, by Talagrand's inequality, there exists a constant $c'>0$ such that
\begin{align*}
\prob{  \check{x}^\top D^* \check{x}  - \eexpect{\check{x}^\top D^* \check{x} } \ge - c'   \log^{3/4} n  }  & \ge 1 - n^{-c}. \\
\prob{ \|(D-E[D]) \check{x} \|  - \eexpect{\|(D-E[D]) \check{x} \|}  \le c'  \log^{3/4} n }  & \ge 1 - n^{-c}.
\end{align*}
Then the desired \prettyref{eq:PSDcheck_unknown} follows by the same argument used in the proof of  \prettyref{thm:SBMSharp_unbalanced}.

\end{proof}

\subsection{Proofs for \prettyref{sec:rary}: Multiple equal-sized clusters}
Theorem \ref{thm:SBMSharp} is proved after three lemmas are given.
 For $k\in [r]$, denote by $C_k \subset [n]$ the support of the $k^\Th$ cluster.
For a set $T$ of vertices, let  $e(i,T) \triangleq  \sum_{j \in T} A_{ij}$
and $e(T', T)= \sum_{i \in T'} e(i, T).$
Let $k(i)$ denote the index of the cluster containing vertex $i.$
Denote the number of neighbors of $i$ in its own cluster by  $s_i = e(i,C_{k(i)})$
and the maximum number of neighbors of $i$ in other clusters by
 $r_i = \max_{k'\neq k(i)} e(i,C_{k'}).$
 \begin{lemma}
\begin{align*}
\prob{ \min_{i\in[n]} ( s_i - r_i  ) \le \log n / \log \log n } \le  r n^{ 1-(\sqrt{a}-\sqrt{b})^2 /r +o(1) }.
\end{align*}
	\label{lmm:fri}
\end{lemma}
\begin{proof}
Notice that $s_i \sim \Binom(K, p)$ and for $k' \neq k(i)$, $e(i, C_{k'}) \sim \Binom(K,q)$.
It is shown in \cite{Abbe14} that
\begin{align*}
\prob{ s_i - e(i, C_{k'} ) \le \log n / \log \log n } \le n^{ -(\sqrt{a}-\sqrt{b})^2 /r +o(1) }.
\end{align*}
It follows from the union bound that
\begin{align*}
\prob{ s_i - r_i  \le \log n / \log \log n } \le  r n^{ -(\sqrt{a}-\sqrt{b})^2 /r +o(1) }.
\end{align*}
Applying the union bound over all possible vertices, we complete the proof.
\end{proof}

\begin{lemma}
There exists a  constant $c>0$ depending only on $b$ and $r$ such that
 \begin{align*}
\prob{ \max_{k \in [r] } \frac{1}{K} \sum_{i \in C_k } r_i \le K q + c \sqrt{\log n }  }  \ge 1- r n^{-2}.
\end{align*}
\label{lmm:ri}
\end{lemma}
\begin{proof}
We first show $\expect{r_i} \le Kq + O(\sqrt{\log n}) $.
Let $t_0 = \sqrt{\log n} $. Then
\begin{align*}
\expect{r_i} & = \expect{ \max_{k' \neq k(i) } e(i, C_{k'})} \\
& = \int_{0}^\infty \prob{  \max_{k' \neq k(i) } e(i, C_{k'}) \ge t } \diff t \\
& \le  \int_{0}^\infty \left( r \prob{  e( i, S_{k'} ) \ge t }  \wedge 1 \right) \diff t \\
& \le Kq+ t_0 + r \int_{ t_0  }^\infty  \prob{  e(i, S_{k'}) -Kq \ge t } \diff t \\
& \overset{(a)}{\le} Kq+ t_0 + r  \int_{ t_0 }^\infty  \exp \left( - \frac{t^2 }{2Kq + 2t /3 } \right) \diff t
\end{align*}
where $(a)$ follows from Bernstein's inequality. Furthermore,
\begin{align*}
\int_{t_0 }^\infty  \exp \left( - \frac{t^2 }{2Kq + 2t /3 } \right) \diff t & \le r \int_{t_0}^{Kq} \eexp^{-3 t^2 / (8Kq) }  \diff t +  \int_{Kq}^{\infty} \eexp^{-3 t /8} \diff t \\
 & \le  \int_{t_0}^{Kq} \eexp^{-3 t_0 t / (8Kq) } \diff t + \frac{8}{3} \eexp^{-3 Kq /8} \\
 & \le \frac{8Kq}{3t_0} \eexp^{-3 t_0^2 / (8Kq) } + \frac{8}{3} \eexp^{-3 Kq /8} =O(\sqrt{\log n}).
\end{align*}
Thus $\expect{r_i} \le Kq +O(\sqrt{\log n})$.
Denote $\sum_{i \in C_k} r_i \triangleq g( A_{ij}, i \in C_k, j \notin C_k )$.
Then $g$ satisfies the bounded difference inequality, \ie,
for all $i =1, 2, \ldots, m$ with $m=(r-1)K^2$,
\begin{align*}
\sup_{x_1, \ldots, x_{m}, x'_i }  | g( x_1, \ldots, x_{i-1}, x_i, x_{i+1}, \ldots, x_{m} ) - g (x_1, \ldots, x_{i-1}, x'_i, x_{i+1}, \ldots, x_{m} ) | \le 1.
\end{align*}
It follows from McDiarmid's inequality that
\begin{align*}
\prob{\sum_{i \in C_k} r_i - \sum_{i \in C_k } \expect{r_i}  \ge K \sqrt{ r \log n}  } \le \exp \left( - \frac{2K^2 r \log n }{K^2 r } \right)  = n^{-2}.
\end{align*}
Thus, with probability at most $n^{-2}$,  $\frac{1}{K} \sum_{i \in C_k} r_i \ge Kq + O(\sqrt{\log n}) $.
The lemma follows in view of the union bound.
\end{proof}


The following lemma provides a deterministic sufficient condition for the success
of \SDP \prettyref{eq:SDP_RZ} in the case $a>b$.
\begin{lemma}  \label{lmm:SDP_R}
Suppose there exist $D^*= \mbox{diag}\{d_i^*\}$ with $d_i^*>0$ for all $i$,
$B^* \in  {\cal S}^n$ with $B^*\geq 0$ and $B_{ij}>0$ whenever $i$ and $j$ are in distinct clusters,  and $\lambda^*\in \reals^n$
such that $S^* \triangleq   D^* - B^* -A + \lambda^*\mathbf{1}^\top +  \mathbf{1}( \lambda^\ast)^\top$ satisfies $S^*\succeq 0$,
and
\begin{eqnarray}
S^*\xi^\ast_k &=&0~~~k \in [r]  \\
B_{ij}^*Z_{ij}^* &=& 0  ~~~~ i,j \in [n]
\end{eqnarray}
Then $\widehat{Z}_{SDP}=Z^*$ is the unique solution to \prettyref{eq:SDP_RZ}.
\end{lemma}


\begin{proof} 
Let $H=Z-Z^*,$ where $Z$ is an arbitrary feasible matrix for the SDP \eqref{eq:SDP_RZ}.
Since  $Z$ and $Z^*$ are both feasible,  $\langle D^* , H \rangle = \langle \lambda^*\mathbf{1}^\top , H\rangle =
 \langle  \mathbf{1}( \lambda^\ast)^\top, H \rangle = 0.$
Since $A=D^* - B^* -S^* + \lambda^*\mathbf{1}^\top +  \mathbf{1}( \lambda^\ast)^\top,$
$$
\langle A, H \rangle =  - \langle B^*, H \rangle - \langle S^*,  H \rangle
$$
and the following hold:
\begin{itemize}
\item
$\langle B^*, H \rangle \geq 0,$ with equality if and only if $\langle B^*,Z\rangle=0.$   That is because
$B^*\geq 0,$  $Z\geq 0,$ and
$\langle B^*, Z^*  \rangle = 0.$
\item
$\langle S^*, H \rangle \geq 0,$ with equality if and only if  $\langle S^*, Z \rangle =0.$  That is because
$\langle S^*, Z \rangle \geq 0$ (because $S^*, Z \succeq 0$) and $\langle S^*, Z ^* \rangle = 0$ (because
$Z^*= \sum_{k=1}^r \xi^*_k (\xi^*_k)^\top $  and $S^*\xi^*_k=0$ for all $k \in [r]$).
\end{itemize}
Thus, $\langle A, H \rangle \leq 0,$  so that $Z^*$ is a solution to the SDP.

To prove that $Z^*$ is the unique solution, restrict attention to the case that $Z$ is another solution to the SDP.
We need to show $Z=Z^*.$   Since both $Z$ and $Z^*$ are solutions, $\langle A, H \rangle=0,$ so that
$  \langle B^*, H \rangle = \langle S^*,  H \rangle = 0.$   Therefore, by the above two
points:  $\langle B^*,Z\rangle=\langle S^*, Z \rangle = 0.$   For each $i$, $B^*_{i,j}=0$ if and only if vertices $i$ and $j$
are in the same cluster.  Also, the fact  $Z\succeq 0$ and $Z_{ii}\leq 1$ for all $i$ implies $Z_{ij}\leq 1$ for all $i,j.$
Thus, the only way $Z$ can meet the constraint $Z \mathbf{1} = K\mathbf{1}$ is that
$Z_{ij}=1$ whenever $i$ and $j$ are in the same cluster.  Therefore $Z=Z^*$ and hence $Z^*$ is the unique solution.
\end{proof}

\begin{proof}[Proof of \prettyref{thm:SBMSharp}]

We now begin the proof of Theorem  \ref{thm:SBMSharp}.
Let $E$ denote the subspace spanned by vectors $\{\xi^\ast_k\}_{k \in [r]}$, \ie, $E = \text{span}(\xi^\ast_k: k\in [r] ).$
Ultimately, we will show that
\begin{equation}
\prob{\inf_{x \Perp E, \|x\|_2=1} x^\top S^\ast x  > 0} \to 1.	
	\label{eq:lambda2}
\end{equation}
Note that $\expect{A}= (p-q) Z^\ast + q \allones - p \identity$ and $Z^\ast= \sum_{k \in [r]}  \xi^\ast_k (\xi^\ast_k)^\top$.
Thus for any $x$ such that $x \Perp E$ and $\|x\|_2=1$,
\begin{align}
x^\top S^\ast x &= x^\top D^\ast x- x^\top \expect{A} x  - x^\top B^\ast x + 2 x^\top \lambda^\ast  \mathbf{1}^\top x - x^\top  \left( A - \expect{A} \right) x \nonumber  \\
&  \overset{(a)}{=} x^\top D^\ast  x  - (p-q) x^\top Z^\ast x - q  x^\top \allones x + p  -x^\top B^\ast x -
x^\top  \left( A- \expect{A} \right) x  \nonumber \\
& \overset{(b)}{=} x^\top D^\ast  x  + p - x^\top B^\ast x - x^\top  \left( A- \expect{A} \right) x  \nonumber \\
& \ge x^\top D^\ast  x + p - x^\top B^\ast x  - \|  A- \expect{A}  \|
 \label{eq:SBMPSDCheck}
\end{align}
where $(a)$ holds because $x \Perp \mathbf{1}$; $(b)$ holds  due to $\iprod{x}{\xi^\ast_k}=0$  for all $k \in [r]$ and $x \Perp \mathbf{1}$.
In view of \cite[Theorem 5]{HajekWuXuSDP14}, $\|  A - \expect{A}\| \leq c' \sqrt{\log n} $
with high probability for a positive constant $c'$ depending only on $a$. 
To bounds \prettyref{eq:SBMPSDCheck} from below, it is convenient to choose $B^\ast$ such that
\begin{align}
x^\top B^\ast x=0, \quad \forall x  \Perp E. \label{eq:orthogonal}
\end{align}
Since $B^*$ is assumed to be symmetric,
\prettyref{eq:orthogonal} is  equivalent to requiring that
\begin{align}
B^\ast_{C_k \times C_{k'} } (i,j) =  B^\ast_{C_{k'}  \times C_{k} } (j,i)   &=   y^\ast_{kk'}(i) + z^\ast_{kk'} (j), \quad \forall 1 \le k< k' \le r  \label{eq:Bconstraint}
\end{align}
for some $y_{kk'} ^\ast$ and $z_{kk'} ^\ast$.    Next we ensure that $S^*\xi^\ast_k=0$ for $k \in [r]$.
Equivalently, we want to ensure that for any distinct $k, k' \in [r]$ and any $i \in C_k$,
 \begin{eqnarray}
 d_i^\ast = e(i, C_k) - \lambda^\ast_i  K - \sum_{i \in C_k} \lambda^\ast_i,   \label{eq:defd}  \\
 e(i,C_{k'}) + \sum_{j\in C_{k'}} (B^\ast_{ij} - \lambda^\ast_j) = K\lambda^\ast_i.   \label{eq:eij}
 \end{eqnarray}
Requiring \eqref{eq:eij} for all distinct $k, k' \in [r]$ and all $i\in C_k$  is equivalent to requiring
\begin{eqnarray}    \label{eq:eji}
  e(j,C_k) + \sum_{i \in C_k} (B^\ast_{ji} - \lambda^\ast_i) = K\lambda^\ast_j
\end{eqnarray}
for all distinct $k,k'\in[r]$ and all $j\in C_{k'}$  (by swapping $i$ for $j$ and $k$ for $k'$).
Moreover, it is equivalent to checking both \eqref{eq:eij} and \eqref{eq:eji} under the
additional assumption that $k < k'.$   Substituting  \eqref{eq:Bconstraint} into
\eqref{eq:eij} and \eqref{eq:eji} gives that for all $k,k' \in [r]$ with $k< k',$
\begin{eqnarray}
e(i, C_{k'} ) + K y^\ast_{kk'}(i) + \sum_{j\in C_{k'}} ( z^\ast_{kk'}(j) - \lambda^\ast_j )= K\lambda^\ast_i ~~~\mbox{for } i\in C_k    \label{eq.eijji} \\
e(j,C_{k}) +  \sum_{i\in C_{k}} ( y^\ast_{kk'}(i) - \lambda^\ast_i)   + Kz^\ast_{kk'}(j)  = K\lambda^\ast_j  ~~~\mbox{for } j\in C_{k'}.    \label{eq.ejiij}
\end{eqnarray}

For $k < k'$ and $i\in C_k,j \in C_{k'}$, set
 \begin{align}
y^\ast_{kk'}(i) = & ~ \frac{1}{K} \pth{ r_i - e(i,C_{k'})} + u_{kk'}	 \\
z^\ast_{kk'}(j) = & ~ \frac{1}{K} \pth{ r_j- e(j,C_{k})} + u_{kk'}
\end{align}
and for $i\in C_k$,
\begin{align}
\lambda^\ast_i = \frac{1}{K} (r_i- \alpha_k).
\end{align}
where $u_{kk'}$ and  $\alpha_k$ are to be determined.
Equations \eqref{eq.eijji} and \eqref{eq.ejiij} both reduce to:
$$
\alpha_k + \alpha_{k'} = \frac{1}{K} e(C_k,C_{k'}) - 2K u_{kk'},
$$
(which must hold whenever $k < k'$) and \eqref{eq:defd} becomes
\[
d^\ast_i = e(i,C_k) - r_i + 2 \alpha_k - \frac{1}{K} \sum_{i \in C_k} r_i .
\]
In view of  \prettyref{lmm:fri} and the assumption that $\sqrt{a}-\sqrt{b} > \sqrt{r}$,
$\min_i (s_i - r_i)  \geq \log n/ \log \log n $ with high probability. By \prettyref{lmm:ri},
$\max_{k \in [r]} \frac{1}{K} \sum_{i \in C_k} r_i \le Kq + O(\sqrt{\log n} ) $ with high probability.
Finally set
\[
\alpha_k = \frac{1}{2} \pth{ Kq - \sqrt{\log n} }, \quad u_{cc'} = \frac{1}{2 K} \sth{\frac{e(C_k,C_{k'})}{K} - Kq + \sqrt{\log n} }.
\]
It follows from the definition that
\begin{align}
\prob{ \min_i d^\ast_i \ge \log n / \log \log n - O (\sqrt{\log n} ) }  \to 1. \label{eq:dlowerbound}
\end{align}
Thus, the desired
\prettyref{eq:lambda2} holds in view of \prettyref{eq:dlowerbound} and \prettyref{eq:SBMPSDCheck}.
Also, note that
$e(C_k, C_{k'} ) \sim \Binom(K^2, q)$.  For $X \sim \Binom(n, p_0)$ with $p_0 \in [0,1]$,
Chernoff's bound yields
\begin{align*}
\prob{ X \le (1-\epsilon) np_0} \le \eexp^{-\epsilon^2 n p_0 /2 }.
\end{align*}
It follows that
\begin{align*}
\prob{e(C_k, C_{k'} ) \le  K^2 q  - K \sqrt{\log n}/2 } \le  n^{-1/(8q) }.
\end{align*}
Applying the union bound, we have that with high probability,  $e(C_k, C_{k'}) > K^2 q  - K \sqrt{\log n} $ for all $1 \le k < k' \le r$.
Hence $y^\ast_{kk'}(i) > 0$ and $z^\ast_{kk'}(j) > 0$ for all $1 \le k < k' \le r$ and $i\in C_k$ and $j\in C_{k'} $ so that $B^\ast_{ij} > 0$ for all $i,j$ in
distinct clusters as desired.
\end{proof}

\begin{proof}[Proof of \prettyref{thm:converse}]
A necessary condition for exact recovery follows from the condition for two clusters.
If a genie were to reveal the membership of all clusters except for clusters 1 and 2, then the exact recovery
problem would be equivalent to recovering two equal sized clusters in a network with $n' = 2K=\frac{2n}{r}$ vertices.
And $p=\frac{(\frac{2a}{r}+o(1))  \log n'}{n'}.$    Similarly, $q=\frac{(\frac{2b}{r}+o(1))  \log n'}{n'}.$
Thus, if $\sqrt{ \frac{2a}{r}  } - \sqrt{  \frac{2b}{r}  } <  \sqrt{2},$
or equivalently if  $\sqrt{a} - \sqrt{b} < \sqrt{r}$,  the converse recovery result of  \cite{Abbe14} implies recovery is
not possible for $r$ equal size clusters.
That is,  if  $a > b$ and $\sqrt{a} - \sqrt{b} < \sqrt{r}$, then for
any sequence of estimators $\widehat{Z}_n$,  $\pprob{ \widehat{Z}_n= Z^*}\rightarrow 0.$
\end{proof}

\subsection{Proofs for \prettyref{sec:cbm}: Binary censored block model}
Our analysis of the \SDP relies on two key ingredients: the spectrum
of labeled \ER random graph and the tail bounds for the binomial distributions,
which we first present.

Recall that $A$ is a symmetric and zero-diagonal random matrix, where the
entries $\{A_{ij}: i<j \}$ are independent and $A_{ij} \sim p(1-\epsilon) \delta_{+1}
+ p \epsilon \delta_{-1} + (1-p) \delta_{0}$ if $i, j$ are in the same cluster; otherwise
$A_{ij}  \sim p(1-\epsilon) \delta_{-1} + p \epsilon \delta_{+1} + (1-p) \delta_{0}$.
Then $\expect{A_{ij}}=p (1-2\epsilon) \sigma_i^\ast \sigma_j^\ast$.
Assume $p \ge c_0 \frac{\log n}{n}$
for any constant $c_0>0$.
We aim to show that $\lnorm{A-\expect{A}}{2} \le c' \sqrt{np}$ with high
probability for some constant $c'>0$.

\begin{theorem}\label{thm:adjconcentration_labeled}
For any $c>0$, there exists $c'>0$ such that for any $n \geq 1$, $\pprob{\lnorm{A-\expect{A}}{2} \le c' \sqrt{n p}} \geq 1-n^{-c}$.
\end{theorem}


\begin{proof}
Let $E=(E_{ij})$ denote an $n\times n$ matrix with independent entries drawn from
$\hat{\mu} \triangleq \frac{p}{2} \delta_1 + \frac{p}{2} \delta_{-1} + (1-p) \delta_0$, which is the distribution of
a Rademacher random variable multiplied with an independent Bernoulli with bias $p$. Define $E'$ as $E'_{ii}=E_{ii}$ and $E'_{ij}=-E_{ji}$ for all $i \neq j$.
Let $A'$ be an independent copy of $A$. Let $D$ be a zero-diagonal symmetric matrix whose entries are drawn from $\hat{\mu}$ and $D'$ be an independent copy of $D$.
Let $M=(M_{ij})$ denote an $n\times n$ zero-diagonal symmetric matrix whose entries are Rademacher and independent
from $C$ and $C'$.
We apply the usual symmetrization arguments:
\begin{align}
\Expect[\|A- \Expect[A]\|]
= & ~ \Expect[\|A - \Expect[A']\|] \overset{(a)}\leq \Expect[\|A-A'\|] \overset{(b)}{=} \Expect[\| (A-A') \circ M \| ] \overset{(c)}{\le} 2\Expect[ \|A \circ M\|] \nonumber \\
= & ~ 2 \Expect[\|D\|]  = 2\Expect[ \|D- \Expect[D'] \| ]
\overset{(d)}\leq 2 \Expect[\|D-D'\|] \overset{(e)}{=}2 \Expect[\|E-E'\| ] \overset{(f)}{\leq} 4 \, \Expect[\|E\|] ,\label{eq:sym}
\end{align}
where $(a),(d)$ follow from  the Jensen's inequality; $(b)$ follows because $A-A'$ has the same distribution as $(A-A')\circ M$, where $\circ$ denotes the element-wise product; $(c),(f)$ follow from the triangle inequality; $(e)$ follows from the fact that $D-D'$ has the same distribution as $E - E'$. In particular, first, the diagonal entries of $D-D'$ and $E-E'$ are all equal to zero.
Second, both $D-D'$ and $E-E'$ are symmetric matrices with independent upper triangular entries.
Third, $D_{ij} -D'_{ij}$ is equal in distribution to $E_{ij}-E'_{ij}$ for all $i<j$ by definition.

Then, we apply the result of Seginer \cite{Seginer00} which characterized the expected spectral norm of i.i.d.\ random matrices within universal constant factors. Let $X_j \triangleq \sum_{i=1}^{n} E_{ij}^2$, which are independent $\Binom(n, p)$. Since $\hat{\mu}$ is symmetric, \cite[Theorem 1.1]{Seginer00} and Jensen's inequality yield
\begin{equation}
\Expect[\|E\|]	\leq \kappa \,\expect{\pth{ \max_{j\in[n]}  X_j }^{1/2} } \leq \kappa \pth{ \expect{\max_{j\in[n]} X_j }}^{1/2}
	\label{eq:seginer}
\end{equation}
for some universal constant $\kappa$.
In view of the following Chernoff bound for the binomial distribution \cite[Theorem 4.4]{Mitzenmacher05}:
\begin{align*}
\prob{X_1 \ge t \log n } \le 2^{- t },
\end{align*}
for all $t \ge 6 np$, setting $t_0=6 \max\{np/\log n, 1\}$ and applying the union bound, we have
\begin{align}
\expect{\max_{j \in [n]} X_j}
= & ~ \int_0^\infty \prob{\max_{j \in [n]} X_j \geq t} \diff t \leq \int_0^\infty (n \, \prob{X_1 \geq t} \wedge 1) \diff t \nonumber \\
\leq & ~ t_0 \log n + n \int_{t_0\log n}^\infty 2^{-t} \diff t \leq (t_0+1) \log n \leq 6(1+2/c_0) np, \label{eq:maxX}
\end{align}
where the last inequality follows from $np \ge c_0\log n$.
Assembling \prettyref{eq:sym} -- \prettyref{eq:maxX},
we obtain
\begin{equation}
	\Expect[\|A - \Expect[A]\|] \leq c_2  \sqrt{np},
	\label{eq:spnorm-mean}
\end{equation}
for some positive constant $c_2$ depending only on $c_0, c_1$.
Since the entries of $A - \Expect[A]$ are valued in $[-1,1]$, Talagrand's concentration inequality for 1-Lipschitz convex functions yields
\[
\prob{\|A - \Expect[A]\| \geq \Expect[\|A - \Expect[A]\|]+t} \leq c_3 \exp(-c_4 t^2)
\]
for some absolute constants $c_3,c_4$, which implies that for any $c>0$,
there exists $c'>0$ depending on $c_0, c_1$, such that $\prob{\|A - \Expect[A]\| \geq c' \sqrt{np}} \leq n^{-c}.$
\end{proof}

Let $X_1, X_2, \ldots, X_m \iiddistr p (1-\epsilon) \delta_{+1}
+ p \epsilon \delta_{-1} + (1-p) \delta_{0}$ for $m \in \naturals$, $p \in [0,1]$ and a fixed constant $\epsilon \in [0, 1/2]$,
where $m=n+o(n)$ and $p=a \log n /n$ for some $a>0$ as $n \to \infty$.
The following upper tail bound for $\sum_{i=1}^m X_i$ follows from the Chernoff bound.
\begin{lemma}\label{lmm:binomialconcentration_labeled}
Assume that $k_n \in [m]$ and $k_n= (1+o(1)) \frac{\log n}{\log \log n}$.
Then
\begin{align}
\prob{ \sum_{i=1}^m X_i  \le k_n } \le  n^{ -  a \left( \sqrt{1-\epsilon}-\sqrt{\epsilon} \right)^2 +o(1)} . \label{eq:chernoffregularupper}
\end{align}
\end{lemma}
\begin{proof}
If $\epsilon=0$, then $\sum_{i=1}^m X_i   \sim \Binom(m,p)$ and the lemma follows from \prettyref{lmm:binomialmaxminconcentration}.
Next we focus on the case $\epsilon>0$.
It follows from the Chernoff bound that
\begin{align}
\prob{ \sum_{i=1}^m X_i  \le k_n } \le \exp ( - m \ell ( k_n /m ) ),   \label{eq:ChernoffTail}
\end{align}
where $\ell (x) = \sup_{ \lambda \ge 0} - \lambda x - \log \expect{\eexp^{-\lambda X_1}}$.
Since $X_1 \sim p (1-\epsilon) \delta_{+1} + p \epsilon \delta_{-1} + (1-p) \delta_{0}$,
\begin{align*}
 \expect{e^{-\lambda X_1}} = 1 + p \left[ \eexp^{-\lambda} (1-\epsilon) + \eexp^{\lambda} \epsilon - 1  \right].
\end{align*}
Notice that $- \lambda x - \log \expect{\eexp^{-\lambda X_1}}$ is concave in $\lambda$, so it achieves the supremum at $\lambda^\ast$
such that
\begin{align*}
- x +  \frac{ p ( \eexp^{-\lambda^\ast} (1-\epsilon)  - \eexp^{\lambda^\ast} \epsilon)   }{1 + p \left[ \eexp^{-\lambda^\ast} (1-\epsilon) + \eexp^{\lambda^\ast} \epsilon - 1  \right] } =0.
\end{align*}
Hence, by setting $x= k_n/m$, we get $\lambda^\ast = \frac{1}{2 } \log \frac{1-\epsilon}{\epsilon} + o(1)$ and thus
\begin{align*}
\ell ( k_n /m ) &= - \lambda^\ast k_n /m - \log \left(1 + p \left[ \eexp^{-\lambda^\ast} (1-\epsilon) + \eexp^{\lambda^\ast} \epsilon - 1  \right] \right) \\
& = -  \frac{1}{2 } \log \frac{1-\epsilon}{\epsilon} \frac{k_n}{m} - \log \left( 1- p ( \sqrt{1-\epsilon} - \sqrt{\epsilon} )^2 \right) + o( k_n /m ) \\
& =  a   ( \sqrt{1-\epsilon} - \sqrt{\epsilon} )^2  \log n /n  + o( \log n /n),
\end{align*}
where the last equality holds due to the Taylor expansion of $\log (1-x)$ at $x=0$ and $p= a \log n /n$. 
Combining the last displayed equation with  \prettyref{eq:ChernoffTail} gives the desired \prettyref{eq:chernoffregularupper}.
\end{proof}
The following lemma establishes a lower tail bound for $\sum_{i=1}^m X_i$.
\begin{lemma} \label{lmm:binomialconcentrationlowerbound_labeled}
Let $k_n$ be defined in \prettyref{lmm:binomialconcentration_labeled}.
Then
\begin{align*}
\prob{ \sum_{i=1}^m X_i  \le - k_n } \ge  n^{ -  a \left(\sqrt{1-\epsilon}-\sqrt{\epsilon} \right)^2 +o(1)} .
\end{align*}
\end{lemma}
\begin{proof}
Let $k^\ast \triangleq \lfloor 2a \sqrt{\epsilon (1-\epsilon) } \log n \rfloor$. Notice that $\sum_{i=1}^m X^2_i \sim \Binom(m, p)$.
Let $Z_1, Z_2, \ldots, Z_n \iiddistr (1-\epsilon)\delta_{+1} + \epsilon \delta_{-1}.$  Then
\begin{align}
\prob{ \sum_{i=1}^m X_i  \le - k_n } & \ge \prob{\sum_{i=1}^m X_i  \le - k_n \bigg| \sum_{i=1}^m X^2_i  = k^\ast }\prob{\sum_{i=1}^m X^2_i  = k^\ast } \nonumber \\
& \overset{(a)}{=} \prob{\sum_{i=1}^{k^\ast} Z_i  \le - k_n } \prob{\sum_{i=1}^m X^2_i  = k^\ast }  \label{eq:lowertailbound}
\end{align}
where $(a)$ holds because conditioning on $ \sum_{i=1}^m X^2_i  = k^\ast$, $\sum_{i=1}^m X_i$ and $\sum_{i=1}^{k^\ast} Z_i $ have the same distribution.
Next we lower bound $ \prob{\sum_{i=1}^{k^\ast} Z_i  \le - k_n }$ and $\prob{\sum_{i=1}^m X^2_i  = k^\ast}$ separately.

We use the following non-asymptotic bound on the binomial tail probability \cite[Lemma 4.7.2]{ash-itbook}: For $U \sim \Binom(n,p)$,
\begin{align*}
(8 k (1-\lambda))^{-1/2} \exp(- n D(\lambda\|p)) \leq \prob{U \geq k} \leq \exp(- n D(\lambda\|p)),
\end{align*}
where $\lambda = \frac{k}{n} \in (0,1) \ge p$ and   $D(\lambda\|p) = \lambda \log \frac{\lambda}{p} + (1-\lambda) \log \frac{1-\lambda}{1-p}$ is the binary divergence function.
Let $W \sim \Binom(k^\ast, \epsilon)$. Then,
\begin{align*}
\prob{\sum_{i=1}^{k^\ast} Z_i  \le - k_n } = \prob{W \ge \frac{k^\ast+ k_n}{2} } & \geq  \sqrt{\frac{k^\ast}{2(k^\ast +k_n) (k^\ast -k_n)  } }  \exp \left[- k^\ast D \left( \frac{1}{2} + \frac{k_n }{2k^\ast} \bigg\|\epsilon \right) \right] \\
& = \exp \left[ - k^\ast D( 1/2 \| \epsilon )  + o(\log n) \right],
\end{align*}

Moreover,  using the following bound on binomial coefficients \cite[Lemma 4.7.1]{ash-itbook}:
\begin{align*}
\frac{\sqrt{\pi}}{2} \leq \frac{\binom{n}{k}}{(2 \pi n \lambda (1-\lambda))^{-1/2}  \exp(n h(\lambda))} \leq 1  .
\end{align*}
where $\lambda = \frac{k}{n} \in (0,1)$ and  $h(\lambda) = -\lambda \log \lambda - (1-\lambda) \log (1-\lambda)$ is the binary entropy function,
we have that
\begin{align*}
\prob{\sum_{i=1}^m X^2_i  = k^\ast } &= \binom{m}{k^\ast } p^{k^\ast} (1-p)^{k^\ast}  \geq \frac{1}{ 2\sqrt{ 2 k^\ast (1-k^\ast/n) } }    \exp \left(-m D( k^\ast /n \| p ) \right) \\
& = \exp \left( - a \log n+  k^\ast \log \frac{e a  \log n}{k^\ast } + o(\log n)  \right).
\end{align*}
Observe that by the definition of $k^\ast$, $\log \frac{a  \log n}{k^\ast }= D( 1/2 \| \epsilon )  +o(1) $ and it follows from \prettyref{eq:lowertailbound} that
\begin{align*}
\prob{ \sum_{i=1}^m X_i  \le - k_n } \ge  \exp \left[ - a \log n + 2a \sqrt{\epsilon (1-\epsilon) } \log n + o(\log n) \right] = n^{- a (\sqrt{1-\epsilon} - \sqrt{\epsilon} )^2 + o (1)  }.
\end{align*}

\end{proof}

The following lemma provides a deterministic sufficient condition for the success of \SDP \prettyref{eq:SBMconvex_labeled} in the case of $a>b$.
\begin{lemma}\label{lmm:SBMKKT_labeled}
Suppose there exist $D^\ast=\diag{d^\ast_i}$ such that $S^* \triangleq D^\ast-A $ satisfies $S^\ast \succeq 0$, $\lambda_2(S^\ast)>0$ and
\begin{align}
S^\ast \sigma^\ast = 0 . \label{eq:SBMKKT_labeled}
\end{align}
Then $\widehat{Y}_{\SDP}=Y^\ast$ is the unique solution to \prettyref{eq:SBMconvex_labeled}.
\end{lemma}
\begin{proof}
The Lagrangian function is given by
\begin{align*}
L(Y, S, D) = \langle A, Y \rangle + \langle S, Y \rangle - \langle D, Y -\identity \rangle,
\end{align*}
where the Lagrangian multipliers are $S \succeq 0$ and $D=\diag{d_i}$.
Then for any $Y$ satisfying the constraints in \prettyref{eq:SBMconvex_labeled},
\begin{align*}
 \Iprod{A}{Y } \overset{(a)}{\le} L(Y, S^\ast, D^\ast) = \Iprod{D^\ast}{I}
=\Iprod{D^\ast}{Y^\ast}=\Iprod{A+S^\ast}{Y^\ast}\overset{(b)}=\Iprod{A}{Y^\ast},
\end{align*}
where $(a)$ holds because $\Iprod{S^\ast}{Y} \ge 0$; $(b)$ holds because $\Iprod{Y^\ast}{S^\ast}=(\sigma^\ast)^\top S^\ast \sigma^\ast =0$ by \prettyref{eq:SBMKKT_labeled}.
Hence, $Y^\ast$ is an optimal solution. It remains to establish its uniqueness. To this end, suppose $\tY$ is
an optimal solution. Then,
\begin{align*}
\Iprod{S^\ast}{\tY}=\Iprod{D^\ast-A}{\tY} \overset{(a)}{=} \Iprod{D^\ast-A}{Y^\ast} {=}\Iprod{S^\ast}{Y^\ast} =0.
\end{align*}
where $(a)$ holds because $\Iprod{A}{\tY}=\Iprod{A}{Y^\ast}$ and $\tY_{ii}=Y^*_{ii}=1$ for all $i \in [n]$.
In view of \prettyref{eq:SBMKKT_labeled}, since $\tY \succeq 0$, $S^\ast \succeq 0$ with $\lambda_2(S^*)>0$, $\tY$ must be a multiple of $Y^*=\sigma^\ast (\sigma^\ast)^\top$.
Because $\tY_{ii}=1$ for all $i \in [n]$, $\tY=Y^\ast$.
\end{proof}

\begin{proof}[Proof of \prettyref{thm:SBMSharp_labeled}]


Let $D^\ast=\diag{d^\ast_i}$ with
\begin{equation}
d^\ast_i  = \sum_{j=1}^n A_{ij} \sigma^\ast _i \sigma^* _j	.
	\label{eq:di-SBM_labeled}
\end{equation}
It suffices to show that $S^* = D^\ast-A$ satisfies the conditions in \prettyref{lmm:SBMKKT_labeled} with high probability.


By definition, $d^\ast_i \sigma_i^\ast = \sum_{j} A_{ij} \sigma^\ast _j$ for all $i$, \ie, $D^\ast \sigma^\ast =A \sigma^\ast$.
Thus \prettyref{eq:SBMKKT_labeled} holds, that is, $S^*\sigma^* = 0$. It remains to verify that $S^\ast \succeq 0$ and $\lambda_2(S^\ast)>0$ with probability
converging to one, which amounts to showing that
\begin{equation}
\prob{\inf_{x \Perp \sigma^\ast, \|x\|_2=1} x^\top S^\ast x  > 0} \to 1.	
	\label{eq:lambda2_labeled}
\end{equation}
Note that $\expect{A}= (1-2\epsilon) p (Y^\ast - \identity ) $ and $Y^\ast= \sigma^\ast (\sigma^\ast)^\top$. Thus for any $x$ such that $x \Perp \sigma^\ast$ and $\|x\|_2=1$,
\begin{align}
x^\top S^\ast x &= x^\top D^\ast x- x^\top \expect{A} x   - x^\top  \left( A - \expect{A} \right) x \nonumber  \\
&  = x^\top D^\ast  x  -(1-2\epsilon) p \; x^\top Y^\ast x + (1-2\epsilon) p -
x^\top  \left( A- \expect{A} \right) x  \nonumber \\
& \overset{(a)}{=} x^\top D^\ast  x + (1-2\epsilon) p   -
x^\top  \left( A- \expect{A} \right) x  \ge \min_{i\in[n]} d^\ast_i  + (1-2 \epsilon) p - \|  A - \expect{A}\|. \label{eq:SBMPSDCheck_labeled}
\end{align}
where $(a)$ holds since $\iprod{x}{\sigma^\ast}=0$.
It follows from \prettyref{thm:adjconcentration_labeled} that $\|  A - \expect{A}\| \leq c' \sqrt{\log n} $
with high probability for a positive constant $c'$ depending only on $a$.
Moreover, note that each $d_i$ is equal in distribution to $\sum_{i=1}^{n-1} X_i$, where $X_i$ are identically and independently distributed according to $p(1-\epsilon) \delta_{+1}+ p \epsilon \delta_{-1} + 1-p\delta_{0}$.
Hence, \prettyref{lmm:binomialconcentration_labeled} implies that
\begin{align*}
\prob{\sum_{i=1}^{n-1} X_i \ge \frac{\log n}{\log \log n}} \ge 1- n^{- a ( \sqrt{1-\epsilon} - \sqrt{\epsilon} )^2+o(1)}.
\end{align*}
Applying the union bound implies that  $\min_{i\in[n]} d^\ast_i  \ge \frac{\log n}{\log \log n}$ holds with probability at least $1 - n^{1-a ( \sqrt{1-\epsilon} - \sqrt{\epsilon} )^2+o(1)}$. It follows from the assumption $ a ( \sqrt{1-\epsilon} - \sqrt{\epsilon} )^2>1$ and \prettyref{eq:SBMPSDCheck_labeled} that the desired \prettyref{eq:lambda2_labeled} holds, completing the proof.
 \end{proof}

\begin{proof} [Proof of \prettyref{thm:PlantedSharpConverse_labeled}]
The prior distribution of $\sigma^\ast$ is uniform over $\{\pm 1\}^n$.
First consider the case of $\epsilon=0$. If $a<1$, then  the number of isolated vertices tends
to infinity in probability \cite{Erdos59}. Notice that for isolated vertices $i$, vertex
$\sigma^\ast_i$ is equally likely to be $+1$ or $-1$ conditional on the graph.
Hence, the probability of exact recovery converges to $0$.

Next we consider $\epsilon>0$. Since the prior distribution of $\sigma^\ast$ is uniform, the
\ML estimator minimizes the error probability among all estimators and thus we only need to find when the \ML estimator fails.
Let $e(i, T) \triangleq \sum_{j \in T} |A_{ij}|$, denoting the number of edges between vertex $i$ and vertices in set $T\subset [n]$.
Let $s_i= \sum_{j: \sigma^\ast_j=\sigma^\ast_i} A_{ij}$ and $r_i= \sum_{j: \sigma^\ast_j \neq \sigma^\ast_i} A_{ij}$.
Let $F$ denote the event that $\min_{i \in [n] } (s_i - r_i ) \le -1 $.
Notice that $F$ implies the existence of $i \in [n] $ such that $\sigma'$ with $\sigma'_i= - \sigma^\ast_i $ and $\sigma'_j =\sigma^\ast_j$ for $j \neq i$
achieves a strictly higher likelihood than $\sigma^\ast$.
Hence $\prob{\text{\ML fails} } \ge \prob{F}$. Next we bound $\prob{F}$ from the below.

Let $T$ denote the set of first $ \lfloor \frac{n}{\log^2 n} \rfloor$ vertices and $T^c=[n] \backslash T$. Let $s'_i=  \sum_{j \in T^c: \sigma^\ast_j=\sigma^\ast_i } A_{ij}$
and $r'_i= \sum_{j \in T^c: \sigma^\ast_j \neq \sigma^\ast_i} A_{ij}$. Then
\begin{equation}
\min_{i \in [n] } ( s_i - r_i )   \leq \min_{i \in T} ( s_i - r_i )  \leq \min_{i \in T} (  s'_i - r'_i  )  + \max_{i \in T} e(i, T).	
	\label{eq:wa_labeled}
\end{equation}
Let $E_1,E_2$ denote the event that $\max_{i \in T} e(i,T) \le \frac{\log n}{\log \log n} -1$, $\min_{i \in T} ( s'_i  -r'_i  ) \le - \frac{\log n}{\log \log n} $, respectively.
In view of \prettyref{eq:wa_labeled}, we have $F \supset E_1 \cap E_2$ and hence it boils down to proving that $\prob{E_i} \to 1$ for $i=1,2$.

Notice that $e(i, T) \sim \Binom(|T|, a \log n /n )$. In view of the following Chernoff bound for binomial distributions \cite[Theorem 4.4]{Mitzenmacher05}: For $r \ge 1$ and $X \sim \Binom(n,p)$, $\prob{X \ge r np }
\le ( \eexp/r)^{rnp},$
we have
\begin{align*}
\prob{ e(i, T) \ge \frac{\log n}{\log \log n} -1 } \le   \left( \frac{\log^2 n}{a \eexp \log \log n } \right) ^{- \log  n/ \log \log n+1} = n^{-2+o(1)}.
\end{align*}
Applying the union bound yields
\begin{align*}
\prob{E_1} \ge 1 - \sum_{i \in T} \prob{ e(i, T) \ge \frac{\log n}{\log \log n} -1 } \ge
1 - n^{-1+o(1)}.
\end{align*}
Moreover,
\begin{align*}
\prob{E_2} & \overset{(a)}{=} 1- \prod_{i\in T} \prob{ s'_i - r'_i  >  - \frac{\log n}{\log \log n} } \\
&  \overset{(b)}{\ge} 1- \left( 1- n^{- a ( \sqrt{1-\epsilon} - \sqrt{\epsilon}  )^2 + o(1)  } \right)^{|T|}
\overset{(c)}{\ge} 1- \exp \left( - n^{1- a ( \sqrt{1-\epsilon} - \sqrt{\epsilon}  )^2 + o(1) } \right) \overset{(d)}{\to} 1,
\end{align*}
where $(a)$ holds because $\{s'_i - r'_i\}_{ i \in T}$ are mutually independent; $(b)$ follows from \prettyref{lmm:binomialconcentrationlowerbound_labeled};
$(c$) is due to $1+x \le e^x$ for all $x \in \reals$;
$(d$) follows from the assumption that $a ( \sqrt{1-\epsilon} - \sqrt{\epsilon}  )^2 < 1$. Thus $\prob{F} \to 1$  and the theorem follows.
\end{proof}

\subsection{Proofs for \prettyref{sec:general_case}: General cluster structure}

We first present a dual certificate lemma which is useful for the proof of Theorem \ref{thm:SDP_gen}.
Recall that $\xi_k^*$ denotes the indicator vector of cluster $k$ for $k \in [r]$ and $Z^*=\sum_k \xi_k^\top \xi_k.$
\begin{lemma}  \label{lemma:SDP_R}
Suppose there exist $D^*= \diag{d_i^*}$ with $d_i^*>0$ for inlier vertices $i$ and $d_i^*=0$ for outlier vertices $i$,
$B^* \in  {\cal S}^n$ with $B^*\geq 0$ and $B^*_{ij}>0$ whenever $i$ and $j$ belong to distinct clusters,  $\eta^*\in \IR,$  and $\lambda^*\in \IR$
such that $S^* \triangleq   D^* - B^* -A + \eta^*\mathbf{I}+ \lambda^* \mathbf{J}$ satisfies $S^*\succeq 0$,
$\lambda_{r+1}(S^\ast)>0$ (where $\lambda_{r+1}(S^\ast)$ is the $(r+1)^{\Th}$ smallest eigenvalue of $S^*$),
and
\begin{eqnarray}
S^*\xi_k^*&=&0~~~k \in [r] , \\
B_{ij}^*Z_{ij}^* &=& 0  ~~~~ i,j \in [n].
\end{eqnarray}
(If the penalized SDP \eqref{eq.SDP_RZaug} is used, the same $\eta^*$ and $\lambda^*$ should be used in the SDP and in this lemma.)
Then $Z^*$ is the unique solution to both SDP  \eqref{eq.SDP_RZ} and \eqref{eq.SDP_RZaug} (\ie, $\widehat{Z}_{SDP}$ produced
by either SDP is equal to $Z^*$).
\end{lemma}
\begin{proof} 
Let $H=Z-Z^*,$ where $Z$ is either an arbitrary feasible matrix for  the SDP \eqref{eq.SDP_RZ} or  an arbitrary feasible
matrix for the SDP \eqref{eq.SDP_RZaug}.
Since $A- \eta^* \mathbf{I} - \lambda^* \mathbf{J} =D^* - B^* -S^*,$
$$
\langle A, H \rangle  - \eta^*\langle \mathbf{I},H\rangle   -  \lambda^*\langle  \mathbf{J}, H\rangle= \langle D^*, H \rangle - \langle B^*, H \rangle - \langle S^*,  H \rangle
$$
and the following hold:
\begin{itemize}
\item
$\langle D^*, H \rangle \leq 0,$  with equality if and only if $Z_{ii}=1$ for all inlier s $i.$   That
is because for inliers $i$,  $d^*_i > 0$, $Z_{ii} \leq 1 =Z^*_{ii},$ and  for outliers $i,$   $d^*_i=0.$
\item
$\langle B^*, H \rangle \geq 0,$ with equality if and only if $\langle B^*,Z\rangle=0.$   That is because
$B^*\geq 0,$  $Z\geq 0,$ and
$\langle B^*, Z^*  \rangle = 0.$
\item
$\langle S^*, H \rangle \geq 0,$ with equality if and only if  $\langle S^*, Z \rangle =0.$  That is because
$\langle S^*, Z \rangle \geq 0$ (because $S^*, Z \succeq 0$) and $\langle S^*, Z ^* \rangle = 0$ (because
$Z^*$ is a sum of matrices of the form $\xi_k\xi_k^\top$ and $S^*\xi_k=0$ for all $k.$)
\end{itemize}
Thus, $\langle A, H \rangle  - \eta^*\langle \mathbf{I},H\rangle   -  \lambda^*\langle  \mathbf{J}, H\rangle \leq 0.$
Therefore, $Z^*$ is a solution to SDP \eqref{eq.SDP_RZaug}.
If $Z$ is a feasible solution for the SDP \eqref{eq.SDP_RZ},  (as  $Z^*$ is), then
 $\langle \mathbf{I}, H\rangle = \langle \mathbf{J}, H \rangle = 0,$ so we conclude that $\langle A, H \rangle \leq 0,$
so $Z^*$ is also a solution to SDP \eqref{eq.SDP_RZ}.

To prove that $Z^*$ is the unique solution, restrict attention to the case that $Z$ is another solution to either
one of the SDPs.
We need to show $Z=Z^*.$   Since both $Z$ and $Z^*$ are solutions,
$\langle A, H \rangle  - \eta^*\langle \mathbf{I},H\rangle   -  \lambda^*\langle  \mathbf{J}, H\rangle \leq 0,$  so that
$\langle D^*, H \rangle =  \langle B^*, H \rangle = \langle S^*,  H \rangle = 0.$   Therefore, by the above three
points:  $Z_{ii}=1$ for all inliers $i$,  and $\langle B^*,Z\rangle=\langle S^*, Z \rangle = 0.$

Since $B^*_{ij}>0$  whenever $i$ and $j$ are in distinct clusters,  and
$Z\geq 0$ and $B^*\geq 0$, the condition $\langle B^*, Z\rangle=0$
implies that $Z_{ij}=0$ whenever $i$ and $j$ are in distinct clusters.
By assumption, $\xi_k^*$ is an eigenvector of $S^*$ with corresponding
eigenvalue zero, for $1\leq k \leq r.$    Since $\lambda_{r+1}(S)>0$, it follows
that all the other eigenvalues of $S^*$ are strictly positive.   The condition
$\langle S^*, Z \rangle = 0$ thus
implies that all the other eigenvectors of $S^*$ are in the null space of $Z,$
so the eigenvectors of $Z$ corresponding to the positive eigenvalues of $Z$
must be in the span of $\xi^\ast_1, \ldots, \xi^\ast_r.$
It follows that $Z$ is a linear combination of matrices of the form $\xi^\ast_k (\xi^\ast_{k'})^\top,$ for
$k, k' \in [r].$    It follows that $Z_{ij}=0$ if either $i$ or $j$ is an outlier vertex, or both are
outlier vertices. Moreover, whenever $i$ and $j$ are in the same cluster,  $Z_{ij}=Z_{ii}=1$.
In conclusion, $Z=Z^\ast$.
\end{proof}

\begin{proof}[Proof of \prettyref{thm:SDP_gen}]
For $k\in [r]$, denote by $C_k \subset [n]$ the support of the $k^\Th$ cluster.
Also, let $C_0$ denote the set of outlier vertices.  
For a set $T$ of vertices, let  $e(i,T) \triangleq  \sum_{j \in T} A_{ij}$
and $e(T', T)= \sum_{i \in T'} e(i, T).$
Let $k(i)$ denote the index of the cluster containing vertex $i.$
Denote the number of neighbors of $i$ in its own cluster by  $s_i = e(i,C_{k(i)})$
and the maximum number of neighbors of $i$ in other clusters by
 $r_i = \max_{k'\neq k(i)} e(i,C_{k'}).$

Now, let us construct $(D^*, B^*, \eta^*, \lambda^*)$ such that the conditions of Lemma \ref{lemma:SDP_R} hold with high probability.
Notice that $d^*_i=0$ if $i$ is an outlier and $B^*_{C_k \times C_k} =0$ for $k \in [r]$.
In order that $(S^* \xi_k^*)_i=0$ for $i\in C_k$  and $k\in [r]$, we must choose:
$$
d^*_i= \left\{ \begin{array}{cl}
s_i  - \eta^* - \lambda^* K_k  &  i\in C_k, k\in [r] \\
0 & \mbox{$i$ is an outlier}
\end{array} \right..
$$
The condition $S^* \xi_k^*=0$ for $k\in [r]$  also partially constrains the symmetric matrix $B^*.$    We should try to
be economical in the choice of $B^*$ so that we have a chance to prove that $S^* \succeq 0.$

Let
$$
B^*_{C_k\times C_{k'}}(i,j)  =\left\{ \begin{array}{cl}
\lambda^* - \frac{e(i,C_{k'})}{K_{k'}}  - \frac{e(j,C_{k})}{K_{k}}  + \frac{e(C_k,C_{k'})}{K_kK_{k'}} & k\neq k', k,k' \in [r]   \\
\lambda^* - \frac{e(i,C_{k'})}{K_{k'}}   & k=0, k' \in [r]   \\
\lambda^* -  \frac{e(j,C_{k})}{K_{k}}    & k'=0, k \in [r]   \\
0 & k=k'\in \{0,\ldots ,r \}
\end{array} \right. .
$$
Then $B=B^\top,$   $S^*\xi_k^*=0$ for $ k \in [r],$ and $B^*_{ij}Z_{ij}^*\equiv 0.$   It remains to
show $d^*_i >0 $, $B^*_{ij} >0$ whenever $i$ and $j$ are in distinct clusters,  and $S \succeq 0$ for some choice of $\lambda^*$ and $\eta^*.$
Let $E_r=\mbox{span}\{\xi_1, \ldots , \xi_r\}.$  We need to show $x^\top S^* x \geq 0$ for $x\in \IR^n$ with $x\perp E_r.$
A nice thing about the choice of $B^*$ (and it uniquely determined the choice of $B^*$) is that, for $x\perp E_r$,
$$
x^\top B^* x =\sum_{k,k' \in \{ 0, \ldots , r\}, k\neq k'}  \sum_{i\in C_k}\sum_{j\in C_{k'}}   B^*_{C_k\times C_{k'}}(i,j) x_i x_j  =0,
$$
where we used the fact that for each pair of distinct $k$ and $k'$,  each of the terms in the definition of
$B^*_{C_k\times C_{k'}}(i,j) $ is either constant in $i$ or constant in $j$, or both, and if $k=0$ the
terms are constant in $j$ and if $k'=0$ the terms are constant in $i.$
The needed condition $d_i^* \geq 0$ involves getting a lower bound on
the number of edges a vertex $i$ has to other vertices in its own cluster (we can concentrate on the
smallest cluster for that purpose), while the needed
condition $B  \geq 0$ involves an upper bound on the number of edges between a vertex $i$ in
one cluster and the vertices of a different cluster.

Let us next examine conditions to ensure $\lambda_{r+1}(S^\ast)>0.$    We use $\Expect[A]=(p-q)Z^* - p\mathbf{I}_i - q\mathbf{I}_o + qJ$
where $\mathbf{I}_i + \mathbf{I}_o$ is a decomposition of the identity matrix for inlier vertices and outlier vertices.
For any $x \perp E_r,$ we have $x^\top B^* x=x^\top Z^* x=0.$   Therefore, for  any $x \perp E_r$, and
taking $\eta^*= \| A - \Expect[A]\|,$ we have

\begin{eqnarray*}
x^\top S^* x & = &  x^\top D^*x + (\lambda^*-q)x^\top J x + p \sum_{i\in C_1\cup \cdots \cup C_r} x_i^2 + q\sum_{i\in C_0} x_i^2  + \eta^* -x^\top(A-\Expect[A])x  \\
& \ge &  x^\top D^*x + (\lambda^*-q)x^\top J x + p \sum_{i\in C_1\cup \cdots \cup C_r} x_i^2 + q\sum_{i\in C_0} x_i^2  \\
& = &  x^\top D^* x + (\lambda^* -q)  \Bigg(\sum_{i \in C_0} x_i \Bigg)^2  + p \sum_{i\in C_1\cup \cdots \cup C_r} x_i^2 + q\sum_{i\in C_0} x_i^2, 
\end{eqnarray*}
where $\xi_0$ is the indicator function for the set of outlier vertices and we used the
fact that $\allones=\mathbf{1} \mathbf{1}^\top $ and $\mathbf{1}=(\mathbf{1} - \xi_0 )+ \xi_0.$
From this it is clear that if $\lambda^* \geq q,$  then   $\lambda_{r+1}(S^\ast)>0.$
So we will be sure to select $\lambda^* \geq q.$   In fact, that will be needed to ensure that $B_{ij}\geq 0$ for all $i,j.$

It remains to select $\lambda^*$ so that $d_i\geq 0$ and $B_{ij}\geq 0$ for all $i,j$ with high probability.
Let $\lambda^*= \tilde{\tau} \log n/n$ with $\tilde{\tau}= b+ \psi_1+\psi_2$, where $\psi_1$ and $\psi_2$ satisfy the assumptions \prettyref{eq:con1}-\prettyref{eq:con4}.
Then, for inlier vertex $i \in C_k$, in view of \prettyref{lmm:binomialmaxminconcentration} and the definition of $I(\cdot, \cdot)$ in \prettyref{eq:I},
\begin{equation*}
\prob{ s_i \leq \lambda^* K_k  + \frac{\log n}{\log \log n} } \le
n^{- \rho_k I(a, \tilde{\tau} ) + o(1) } \le n^{-\rho_{r} I(a, \tilde{\tau}) + o(1) }.
\end{equation*}
Applying the union bound yields that with probability at least $1-n^{ 1- \rho_r I(a, \tilde{\tau}) + o(1 ) }$,
for $1\leq k \leq r,$
$
\min_{i \in C_k} s_i  \ge \lambda^* K_k  + \log n/\log \log n.
$
The matrix concentration inequality given in \cite[Theorem 5]{HajekWuXuSDP14} shows that $\eta^\ast=||A-\expect{A} || = O(\sqrt{\log n} )$ with high probability.
Therefore, by the assumption $\rho_{r}  I(a, \tilde{\tau}) >1$ and the definition of $d_i^\ast$, it follows that with high probability, $\min_{i \notin C_0} d^\ast_i >0$.

Turning next to $B^\ast_{ij}$'s for $(i,j)\in C_k\times C_{k'},$  it suffices to consider the two
following cases:

\paragraph{Case 1: $k$ and $k'$ correspond to the smallest and second smallest clusters, $r-1$ and $r.$}
Note that   $\frac{e(C_k,C_{k'})}{K_kK_{k'}}$  will be
very close to $q$ with high probability,
so we can replace it by $q=\frac{b\log n}{n}$,
which is also the mean of   $\frac{e(i,C_{k'})}{K_{k'}}$   and  $ \frac{e(j,C_{k})}{K_{k}} .$
Specifically, it follows from the Chernoff bound that
\begin{align*}
\mathbb{P}\left\{
\frac{e(C_k,C_{k'} ) } {K_k K_{k'} }  \le
q- \frac{2 \sqrt{q \log n} } {\sqrt{K_k K_{k'} }  }
\right \}
= \mathbb{P} \left\{ e(C_k,C_{k'} ) \le \mu (1-\epsilon) \right \} \le e^{-\epsilon^2 \mu /2} = n^{-2},
\end{align*}
where $\mu= q K_k K_{k'}$ and $\epsilon= \frac{2 \sqrt{ \log n} }{\sqrt{q K_k K_{k'} }  }.$
In view of \prettyref{lmm:binomialmaxminconcentration}  and the union bound,
\begin{align*}
\prob{ \max_{i\in C_{r-1}} e(i,C_{r})    \geq   (b+ \psi_1)  K_r \log n/n - \frac{\log n}{\log \log n} } & \le  n^{1- \rho_r I(b, b+\psi_1) + o(1) },  \\
\prob{  \max_{i\in C_{r}} e(i,C_{r-1})   \geq  (b+ \psi_2)  K_{r-1} \log n/n - \frac{\log n}{\log \log n}  } & \le n^{1- \rho_{r-1} I(b, b+\psi_2) + o(1) }.
\end{align*}
By the assumptions $\rho_r I(b, b+\psi_1)>1$ and $\rho_{r-1} I(b, b+\psi_2) >1$, it follows that with high probability
$B^\ast_{C_k \times C_{k'}} >0$.

\paragraph{Case 2: $k$ corresponds to the smallest cluster and $C_{k'}$ is the set of outliers (i.e. $k=r, k'=0.$)}
In view of \prettyref{lmm:binomialmaxminconcentration} and the union bound,
\begin{align*}
\prob{ \max_{i\in C_{0}} e(i,C_{r})    \geq   \tilde{\tau}  K_r \log n/n } & \le  n^{1- \rho_r I(b, \tilde{\tau} ) + o(1) }.
\end{align*}
By the assumptions $\rho_r I(b, \tilde{\tau})>1$, it follows that with high probability
$B^\ast_{C_k \times C_{k'}} \ge 0$.

In conclusion, we have constructed $(D^*, B^*, \eta^*, \lambda^*)$  such that the conditions of Lemma \ref{lemma:SDP_R} hold with high probability.
Therefore, the theorem follows by applying Lemma \ref{lemma:SDP_R}.
\end{proof}
%

\begin{lemma}  Let $\tau = \frac{a-b}{\log(a/b)}$ for $0 < a < b.$
Then $ I(b,\tau) \leq (\sqrt{a}- \sqrt{b})^2 \leq 2I(b,\tau).$  \label{lemma.Iab}
\end{lemma}
\begin{proof}
Notice that $I(a,x)+I(b,x)$ is strictly convex in $x$. By setting the derivative
to be zero, we find that it achieves its minimum
value, $(\sqrt{a} - \sqrt{b})^2,$  at $x=\sqrt{ab}.$
By definition, $I(a,\tau)=I(b,\tau)$  and $\sqrt{ab} \leq \tau.$  Thus,
 $(\sqrt{a} -\sqrt{b})^2 \leq I(a,\tau)+I(b,\tau)=2I(b,\tau)$.
 Moreover, $I(a,x)$ is decreasing for $x \le a$ and $I(b,x)$ is non-negative.
 Therefore, $I(b,\tau)=I(a,\tau)\leq I(a,\sqrt{ab}) \leq (\sqrt{a} - \sqrt{b})^2 .$
 \end{proof}

\begin{lemma}  \label{lemma.Ibnd}
For any $\mu >0$ and $x > 0,$   $I(\mu, \mu+ 2x) \leq  4 I(\mu, \mu+ x).$
\end{lemma}
\begin{proof}
Let $f(x)=I(\mu, \mu+ x)$ for $x \geq 0$.   Then $f(0)=f'(0)=0$ and
$f''(s) = \frac{1}{\mu+s}.$  Therefore,
$$
f(x) =  \int_0^x  \int_0^y f''(s) ds dy =
\int_0^x  \int_0^y \frac{1}{\mu+s}  ds dy  = \int_0^x  \int_s^x \frac{1}{\mu+s}  dy ds =  \int_0^x \frac{x-s}{\mu+s} ds.
$$
Thus, using a change of variables $s=2t$,
$$
f(2x)  = \int_0^{2x} \frac{2x-s}{\mu+s} ds =  4  \int_0^x  \frac{x-t}{\mu+2t} dt.
$$
Comparing the expressions for $f(x)$ and $f(2x)$ completes the proof.
\end{proof}

\appendices
\section{ Behavior of threshold function in \prettyref{eq:threshold}}
	\label{app:eta}
Recall $\eta(\rho,a,b)$ defined in \prettyref{eq:threshold} which governs the sharp recovery threshold for the asymmetric binary SBM.
The following lemma implies $\eta(\rho,a,b)$ is minimized at $\rho=1/2.$
\begin{lemma}  \label{lmm:eta_behavior}
For any $a>b>0$, $\eta(\rho, a, b)$ is convex in $\rho$ over $[0,1],$ and symmetric about $\rho=1/2.$
\end{lemma}
\begin{proof}
 Recall that
$$
\eta(\rho, a, b)=\frac{a+b}{2}  - \gamma + \frac{(1-2\rho) \tau}{2} \log  \frac{(\gamma+(1-2\rho) \tau)\rho}{(\gamma- (1-2\rho) \tau)(1-\rho)}
$$
and from this expression it is easily checked that $\eta$ is symmetric about $\rho=1/2.$

Let $\eta', \eta''$ denote the first-order and second-order derivative of $\eta $ with respect to $\rho$, respectively.
We show that $\eta'' \ge 0$. Recall that $\gamma= \sqrt{ (1-2\rho)^2 \tau^2 + 4 \rho (1-\rho) a b }$. Hence,
\begin{align*}
\frac{\diff \gamma }{\diff \rho } = \frac{- 4 (1-2\rho) \tau^2 + 4 (1-2\rho) a b }{2\gamma}  = 2 (1-2\rho) \frac{ab - \tau^2}{\gamma}.
\end{align*}
Let $h(\rho)= \log  \frac{(\gamma+(1-2\rho) \tau) \rho}{(\gamma- (1-2\rho) \tau)(1-\rho)}$ and then
\begin{align*}
\frac{\diff h }{\diff \rho }  & = \frac{1}{\rho} + \frac{1}{1-\rho} + \frac{\diff \gamma / \diff \rho - 2 \tau }{\gamma+(1-2\rho) \tau) \rho}-
\frac{\diff \gamma / \diff \rho + 2 \tau }{\gamma-(1-2\rho) \tau) \rho} \\
&= \frac{1}{\rho(1-\rho) } - \frac{\diff \gamma }{\diff \rho } \frac{2 (1-2\rho) \tau}{4 \rho (1-\rho) a b} - \frac{ 4 \tau \gamma }{4\rho (1-\rho) a b} \\
& = \frac{\gamma - \tau }{\rho (1-\rho) \gamma }.
\end{align*}
It follows that
\begin{align*}
\eta' &=  -   \frac{\diff \gamma }{\diff \rho }  +  \frac{(1-2\rho) \tau}{2} \frac{\diff h }{\diff \rho } - \tau h \\
& =  - 2 (1-2\rho) \frac{ab - \tau^2}{\gamma} +  \frac{(1-2\rho) \tau}{2} \frac{\gamma - \tau }{\rho (1-\rho) \gamma } - \tau h \\
& \overset{(a)}{=} \frac{(1-2\rho) (\tau - \gamma) }{ 2 \rho (1-\rho)} - \tau h.
\end{align*}
where $(a)$ follows using the expression of $\gamma$. Therefore,
\begin{align*}
\eta'' &=  - \frac{1}{2 } \left(  \frac{1}{\rho^2} + \frac{1}{(1-\rho)^2}  \right) (\tau -\gamma)  - \frac{(1-2\rho)  }{ 2 \rho (1-\rho )}  \frac{\diff \gamma }{\diff \rho } - \tau \frac{\diff h }{\diff \rho }  \\
& = - \frac{\rho^2 + (1-\rho)^2 }{2 \rho^2 (1-\rho)^2 } (\tau - \gamma) + \frac{(1-2\rho)^2  (\tau^2- ab) }{\rho(1-\rho) \gamma} + \frac{\tau (\tau-\gamma)}{\rho(1-\rho) \gamma} \\
& = \frac{1}{ \rho(1-\rho) \gamma} \left[   \frac{\rho^2 + (1-\rho)^2}{2 \rho (1-\rho) } \gamma^2  - \left(\frac{\rho^2 + (1-\rho)^2}{2 \rho (1-\rho) } +1 \right) \gamma \tau +  \left( 1+ (1- 2\rho)^2 \right) \tau^2 \right] \\
& = \frac{1}{ \rho(1-\rho) \gamma} \left[   a b - \frac{\gamma \tau }{2 \rho (1-\rho) } + \frac{\rho^2 + (1-\rho)^2 }{2 \rho (1-\rho) } \tau^2 \right] \\
& = \frac{1}{ 2\rho^2(1-\rho)^2  \gamma} \left[    \left( \rho^2 + (1-\rho)^2 \right)  \tau^2 - \gamma \tau + 2 \rho (1-\rho) a b\right]  \\
& = \frac{\tau^2}{ 2\rho^2(1-\rho)^2  \gamma } \left[     \rho^2 + (1-\rho)^2 - \sqrt{(1- 2\rho)^2 + 4 \rho (1-\rho) ab /\tau^2 }+ 2 \rho (1-\rho) a b/\tau^2\right] \ge 0,
\end{align*}
where the last inequality follows because by letting $x = ab/\tau^2$,
\begin{align*}
 \left( \rho^2 + (1-\rho)^2 + 2 \rho (1-\rho) x  \right)^2  -  \left[  (1- 2\rho)^2 + 4 \rho (1-\rho) x \right] =
 4 \rho^2 (1-\rho)^2 (x-1)^2 \ge 0.
\end{align*}
Thus $\eta$ is convex in $\rho.$
\end{proof}

\section{A data-driven choice of the penalization parameter in \prettyref{eq:SDP2}}
	\label{app:thmSDP2}
Fix $a>b$ such that $\sqrt{a}> \sqrt{b}+\sqrt{2}$ and fix $\rho \in (0,\frac{1}{2}]$.
Recall that $p= a \log n/n$, $q= b \log n/n$, $K=\lceil \rho n \rceil$, and $\bar{\rho}=1-\rho$.
Let $d_i = \sum_j A_{ij}$ denote the degree of the $i^\Th$ vertex and set $w_i = \frac{d_i}{\log n}$.
Set	$w_- = a \rho+\bar \rho b$ and $w_+=a \bar \rho+ \rho b$.
Then $\Expect[w_i]=w_- + O(1/n)$ or $w_+ +O(1/n) $ if $\sigma_i=1$ or $-1$, and $w_+ \geq w_-$ with equality if and only if $\rho = \frac{1}{2}$.
Let $\hat{w}=\frac{1}{n} \sum w_i = \frac{2}{n \log n} \sum_{i<j} A_{ij}$, where $\sum_{i<j} A_{ij}$ is distributed as $\Binom\big( \binom{K}{2} + \binom{n-K}{2} , p \big)$ convolved with $\Binom( K(n-K), q )$.
It follows from Bernstein's inequality that
for any $c > 0$, there exists a constant $c'>0$ such that with probability at least $1-n^{-c}$,
\begin{align*}
\bigg| \sum_{i<j}  ( A_{ij} -   \expect{A_{ij} } ) \bigg| \le c' \sqrt{ n } \log n.
\end{align*}
Thus $\hat{w}=  \rho w_{-}+ \bar{\rho} w_{+} + \Oprob(n^{-1/2} ) $.

Set	$\widehat \rho = \frac{1}{n} \sum \indc{w_i \le \hat w}$, $\widehat{w}_{+} = \frac{1}{n} \sum w_i \indc{w_i > \hat w}$ and $\widehat{w}_{-} = \frac{1}{n} \sum w_i \indc{w_i < \hat w}$, which are consistent estimates for $\rho,w_+,w_-$, respectively. From these we can readily obtain consistent estimates for $(a,b,\rho)$ whenever $\rho \neq 1/2$.
Furthermore, when $\rho=1/2$,  we claim that
	\begin{equation}
	\widehat \rho = \frac{1}{2} + \oprob(\log^{-1/9} n).
	\label{eq:rhoconcentrate}
\end{equation}
Now we are ready to choose the penalty parameter $\widehat \lambda = \widehat \lambda(A)$, so that \prettyref{thm:SDP2} continues to hold upon replacing the deterministic $\lambda^*$ by $\hat \lambda$.	
	Let
\begin{equation}
	\hat \lambda = \begin{cases}
	\hat w \frac{\log n}{n} & |\widehat\rho - \frac{1}{2}| \leq \log^{-1/9} n\\
	\frac{\widehat{w}_+ - \widehat{w}_-}{1 - 2 \widehat\rho} \frac{1}{\log \frac{(\widehat{w}_{+} + \widehat{w}_-)(1 - 2 \widehat\rho) + (\widehat{w}_+ - \widehat{w}_-)}{(\widehat{w}_+ + \widehat{w}_-)(1 - 2 \widehat\rho) - (\widehat{w}_+ - \widehat{w}_-)}}  \frac{\log n}{n} & |\widehat\rho - \frac{1}{2}| > \log^{-1/9} n.
	\end{cases}
	\label{eq:lambdahat}
\end{equation}
To verify the correctness of the SDP, it suffices to show that $\widehat \lambda$ is close to the appropriate deterministic penalty term in probability.	
First consider $\rho = \frac{1}{2}$. Then \prettyref{eq:rhoconcentrate} implies that $\widehat \lambda = \hat w \frac{\log n}{n} =  (\frac{a+b}{2} + \oprob(1)) \frac{\log n}{n}$. The proof of \cite[Theorem 2]{HajekWuXuSDP14} for the binary symmetric SBM shows that any $\hat{\lambda} \geq q$ suffices.	
Next consider $\rho \neq \frac{1}{2}$. Then $\rho \in (\epsilon,\frac{1}{2}-\epsilon)$ for some $\epsilon > 0$. Since $\hat{\rho}$ is a consistent estimator of $\rho$ when $\rho \neq 1/2$, it follows that
$\widehat \lambda$ is set according to the second case of \prettyref{eq:lambdahat}.	Recall that $\lambda^* = \tau \frac{\log n}{n}$ and $\tau = \frac{a-b}{\log a-\log b}$.	Then $\widehat \lambda = (\tau + \oprob(1)) \frac{\log n}{n}$ and the proof of \prettyref{thm:SDP2} carries over.

It remains to prove \prettyref{eq:rhoconcentrate}.
To this end, note that with high probability,  $\hat w = \frac{a+b}{2} + \Oprob(n^{-1/2})$.
Set $\delta_n = n^{-1/3}$ and define $\widehat \rho_{\pm} = \frac{1}{n} \sum \indc{w_i \le (a+b)/2 \pm \delta_n}$. Since $\widehat \rho_{-} \leq \widehat \rho \leq \widehat \rho_{+}$ with probability tending to one, it suffices to show both $\widehat \rho_{\pm}$ satisfy \prettyref{eq:rhoconcentrate}. We only consider $\widehat \rho_-$ as the other case follows entirely analogously.
Define $X= (w_1-\expect{w_1})/\sqrt{\var(w_1)}$.
Denote by $F_{X}$ and $\Phi$ the cumulative distribution function of $w_1$ and the standard normal distribution, respectively.
It follows from Berry-Esseen inequality that
\begin{align*}
\sup_{x \in \reals} | F_{X} (x) - \Phi(x) | =O(1/\sqrt{\log n}).
\end{align*}
Since $\expect{w_1} = \frac{a+b}{2}  + O(\frac{1}{n})$ and $\var(w_1) = O(1/\log n)$, we have
\begin{align}
\expect{\widehat \rho_- \; } & = \prob{w_1 \le \frac{a+b}{2} - \delta_n}= F_{X} \left( O(\delta_n \sqrt{\log n} ) \right) \nonumber \\
&= \Phi \left( O(\sqrt{\log n}/n^{1/3} ) \right) + O(1/\sqrt{\log n}) = 1/2+ O(1/\sqrt{\log n}).  \label{eq:be}
\end{align}
Moreover, by definition
\begin{align}
\var(\hat \rho_- )=  \frac{1}{n} \prob{w_1 \le \frac{a+b}{2} - \delta_n}   + \frac{n(n-1)}{n^2}  \prob{w_1, w_2 \le \frac{a+b}{2} - \delta_n} - \mathbb{P}^2 \left \{ w_1 \le \frac{a+b}{2} - \delta_n \right\} . \label{eq:defhatrho}
\end{align}
Let $\tilde{d}_1=\sum_{j \neq 2} A_{1j}$ and $\tilde{d}_2=\sum_{j\neq 1} A_{2j}$.  Let $\tilde{w}_i=\tilde d_i/\log n$ for $i=1,2$. Then
\begin{align}
 \prob{w_1, w_2 \le \frac{a+b}{2} - \delta_n} \le  \prob{ \tilde{w}_1, \tilde{w}_2 \le \frac{a+b}{2} - \delta_n } = \mathbb{P}^2 \left \{ \tilde{w}_1 \le \frac{a+b}{2} - \delta_n  \right\}, \label{eq:boundw}
\end{align}
where the last equality holds because $\tilde{w}_1$ and $\tilde{w_2}$ are independent and identically distributed. Similarly to \prettyref{eq:be}, we have
$\prob{ \tilde{w}_1 \le \frac{a+b}{2} - \delta_n  } = 1/2+ O(1/\sqrt{\log n})$.
Therefore, in view of \prettyref{eq:defhatrho} and \prettyref{eq:boundw},
we have that $\var(\hat \rho_- )=  O(1/\sqrt{\log n})$.
By Chebyshev's inequality, with probability at least $1-\log^{-1/4} n$, $| \hat \rho_- - \expect{\widehat{\rho}_- \; }| \le O( \log^{-1/8 } n )$, completing the proof.

\bibliographystyle{abbrv}
\bibliography{../../one_community/graphical_combined}
\end{document}